\documentclass[letterpaper]{article} 

\PassOptionsToPackage{numbers, compress, sort}{natbib}

\usepackage{times}  
\usepackage{helvet}  
\usepackage{courier}  
\usepackage{url}  
\usepackage{graphicx}  
\usepackage{natbib}
\usepackage[utf8]{inputenc} 
\usepackage[T1]{fontenc}    
\usepackage{hyperref}       
\usepackage{url}            
\usepackage{booktabs}       
\usepackage{amsfonts}       
\usepackage{amsmath}  
\usepackage{amssymb}  
\usepackage{nicefrac}       
\usepackage{microtype}      
\usepackage{bbm}
\usepackage{tikz}
\usetikzlibrary{shapes.misc, positioning}
\usetikzlibrary{fit,positioning,arrows,automata,calc}
\tikzset{
	main/.style={circle, minimum size = 5mm, thick, draw =black!80, node distance = 10mm},
	connect/.style={-latex, thick},
	box/.style={rectangle, draw=black!100}
}
\usepackage{thmtools,thm-restate}
\usepackage{parskip}
\usepackage{authblk}
\usepackage[linesnumbered,ruled,noend]{algorithm2e}
\usepackage{algpseudocode}
\usepackage[subtle]{savetrees}

\usepackage{floatrow}
\newfloatcommand{capbtabbox}{table}[][\FBwidth]

\usepackage{macros}
\usepackage{basic}

\newcommand\num[1]{\textcolor{black}{{#1}}}
\newcommand\sn{\texttt{Dugong}} 

\newenvironment{proof}{\paragraph{Proof:}}{\hfill$\square$}

\newcommand\tradimp{$16.0$}

\newcommand\dpimp{$24.2$}

\newcommand\depsimp{$10.2$}
\newcommand\shareimp{$9.2$}
\newcommand\priorimp{$13.7$}
\newcommand\priorimpdev{$1.7$}

\usepackage{etoolbox}

\makeatletter
\newcommand{\changeoperator}[1]{%
  \csletcs{#1@saved}{#1@}%
  \csdef{#1@}{\changed@operator{#1}}%
}
\newcommand{\changed@operator}[1]{%
  \mathop{%
    \mathchoice{\textstyle\csuse{#1@saved}}
               {\csuse{#1@saved}}
               {\csuse{#1@saved}}
               {\csuse{#1@saved}}%
  }%
}
\makeatother

\changeoperator{sum}
\changeoperator{prod}

\let\oldnl\nl
\newcommand{\nonl}{\renewcommand{\nl}{\let\nl\oldnl}}

\title{Multi-Resolution Weak Supervision \\for Sequential Data}

\author[$\dagger$]{Frederic~Sala$^*$}
\author[$\dagger$]{Paroma~Varma\thanks{Equal Contribution}}
\author[$\dagger$]{Jason~Fries}
\author[$\dagger$]{Daniel~Y.~Fu}
\author[$\dagger$]{Shiori~Sagawa}
\author[$\dagger$]{Saelig~Khattar}
\author[$\dagger$]{Ashwini~Ramamoorthy}
\author[$\ddagger$]{Ke~Xiao}
\author[$\dagger$]{Kayvon~Fatahalian}
\author[$\dagger$]{James~Priest}
\author[$\dagger$]{Christopher~R{\'e}}

\affil[$\dagger$]{Stanford University}
\affil[$\ddagger$]{University of Massachusetts Amherst}
\affil[ ]{\footnotesize{\texttt{\{fredsala, paroma, jfries, danfu, sagawas, saelig, ashwinir, \. }}}
\affil[ ]{\footnotesize{\texttt{\. kayvonf, jpriest, chrismre\}@stanford.edu, kexiao@cs.umass.edu}}}

\begin{document}

\maketitle

\begin{abstract}

Since manually labeling training data is slow and expensive, recent industrial and scientific research efforts have turned to \emph{weaker} or noisier forms of supervision sources.
However, existing weak supervision approaches fail to model \emph{multi-resolution} sources for sequential data, like video, that can assign labels to individual elements or collections of elements in a sequence. 
A key challenge in weak supervision is estimating the unknown accuracies and correlations of these sources without using labeled data. Multi-resolution sources exacerbate this challenge due to complex correlations and sample complexity that scales in the length of the sequence.
We propose \sn{}, the first framework to model multi-resolution weak supervision sources with complex correlations to assign probabilistic labels to training data.  
Theoretically, we prove that \sn{}, under mild conditions, can uniquely recover the unobserved accuracy and correlation parameters and use parameter sharing to improve sample complexity. 
Our method assigns \emph{clinician-validated} labels to population-scale biomedical video repositories, helping outperform traditional supervision by $36.8$ F1 points and addressing a key use case where machine learning has been severely limited by the lack of expert labeled data.
On average, \sn{} improves over traditional supervision by \tradimp{} F1 points and existing weak supervision approaches by \dpimp{} F1 points across several video and sensor classification tasks.

\end{abstract}

\section{Introduction}
\label{sec:intro}


Many machine learning models rely on a large amount of labeled data for their success. However, since hand-labeling training sets is slow and expensive, domain experts are turning to \emph{weaker}, or noisier forms of supervision sources like heuristic patterns~\cite{hearst1992automatic}, distant supervision~\cite{mintz2009distant}, and user-defined programmatic functions~\cite{ratner2016data} to generate training labels.
The goal of \emph{weak supervision} frameworks is to automatically generate training labels to supervise arbitrary machine learning models by estimating unknown source accuracies~\cite{Ratner16,xiao2015learning,takamatsu2012reducing,khetan2017learning,guan2018said,zhan2019sequentialws}.

Practitioners can therefore leverage the power of complex, discriminative models without hand-labeling large training sets by encoding domain knowledge in supervision sources.
This has achieved state-of-the-art performance in many applications ~\cite{niu2012deepdive,xiao2015learning} and has been deployed by several large companies~\cite{bach2018snorkel,dehghani2017learning,dehghani2017neural,liang2016neural,mahajan2018exploring,jia2017constrained}. 
However, current weak supervision techniques do not account for sources that assign labels at \emph{multiple resolutions} (e.g. labeling individual elements and collections of elements), which is common in sequential modalities like sensor and video.

Consider training a deep learning model to detect interviews in TV news videos~\cite{fu2019rekall}. As shown in Figure~\ref{fig:multigranular}, supervision sources to generate \emph{training labels} can draw on indirect signals from closed caption transcripts (per-scene), bounding box movement between frames (per-window), and pixels in the background of each frame (per-frame). 
However, existing weak supervision frameworks cannot model two key aspects of this style of sequential supervision.
First, sources are \emph{multi-resolution} and can assign labels on a per-frame to per-window to per-scene basis, implicitly creating \emph{sequential correlations} among the noisy supervision sources that can lead to conflicts within and across resolutions. 
Second, we have no principled way to incorporate \emph{distribution prior}, like how frames with interviews are distributed within a scene.

\begin{figure}
  \includegraphics[width=\columnwidth]{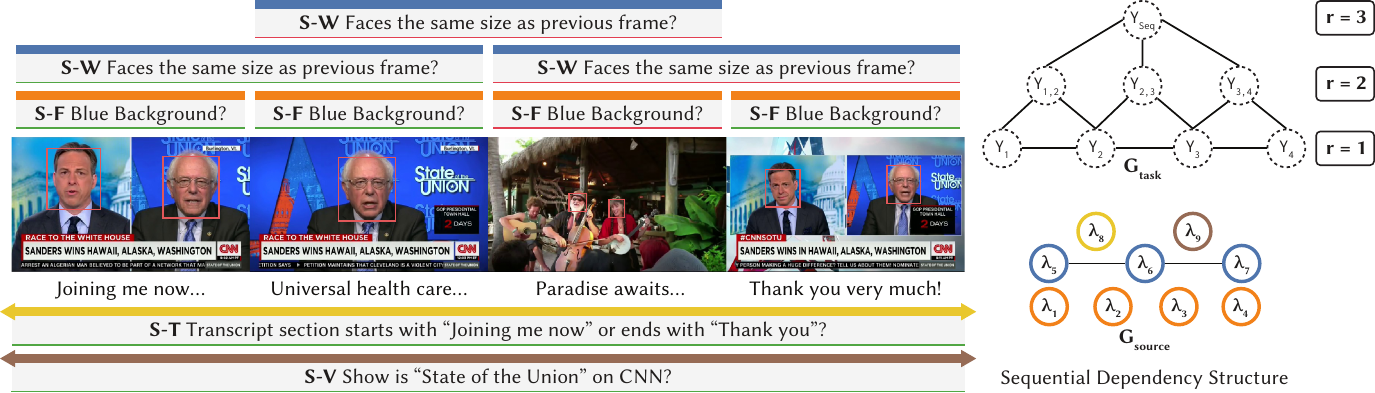}
    \caption{Multi-resolution weak supervision sources to label video analytics training data. S-X outputs noisy label vectors $\lambda_j$ and represents various supervision sources at different resolutions: Video (V), Transcript (T), Window (W), and Frame (F). We show the graphical model structure for modeling these sources at different resolutions: dotted nodes represent latent true labels, solid nodes represent the noisy supervision sources, and edges represent sequential relations.}
  \label{fig:multigranular}
\end{figure}

The core technical challenge in this setting is integrating diverse sources with unknown correlations and accuracies at scale \emph{without observing any ground truth labels}. 
Traditionally, such issues have been tackled via probabilistic graphical models, which are expressive enough to capture sequential correlations in data.
Unfortunately, learning such models via classical approaches such as variational inference~\cite{wainwright2008graphical} or Gibbs sampling~\cite{koller2009probabilistic} presents both practical and theoretical challenges: these techniques often fail to scale, in particular in the case of long sequences. 
Moreover, algorithms for latent-variable models may not always converge to a unique solution, especially in cases with complex correlations.

We propose \sn --- the first weak supervision framework to integrate multi-resolution supervision sources of varying quality and incorporate distribution prior to generate high-quality training labels.
Our model uses the agreements and disagreements among diverse supervision sources, instead of traditional hand-labeled data, at different resolutions (e.g., frame, window, and scene-level) to output probabilistic \emph{training labels} at the required resolution for the end model.
We develop a simple and scalable approach that estimates parameters associated with source accuracy and correlation by solving a pair of linear systems. 

We develop conditions under which the underlying statistical model is identifiable. With mild conditions on the correlation structure of sources, we prove that the model parameters are recoverable directly from the systems. 
We show that we can reduce the dependence of sample complexity on the length of the sequence from exponential to linear to independent, using various degrees of parameter sharing, which we analyze theoretically. 
Applying recent results in weak supervision literature, we then show that the generalization error of the end model scales as $O(1/\sqrt{n})$ in the number of unlabeled data points---the same asymptotic rate as supervised approaches.

We experimentally validate our framework on five real-world sequential classification tasks over modalities like medical video, gait sensor data, and industry-scale video data. For these tasks, we collaborate with domain experts like cardiologists to create multi-resolution weak supervision sources.
Our approach outperforms traditional supervision by \tradimp{} F1 points and existing state-of-the-art weak supervision approaches by \dpimp{} F1 points on average.

We also create an SGD variant of our method that enables implementation in modern frameworks like PyTorch and achieves $90\times$ faster runtimes compared to prior Gibbs-sampling based approaches~\cite{ratner2016data,bach2017learning}.
This scalability enables using clinician-generated supervision sources to automatically label population-scale biomedical repositories such as the UK Biobank~\cite{sudlow2015uk} on the order of days, addressing a key use case where machine learning has been severely limited by the lack of expert labeled data and improving over state-of-the-art traditional supervision by $36.8$ F1 points.


\begin{figure}
 \includegraphics[width=\columnwidth]{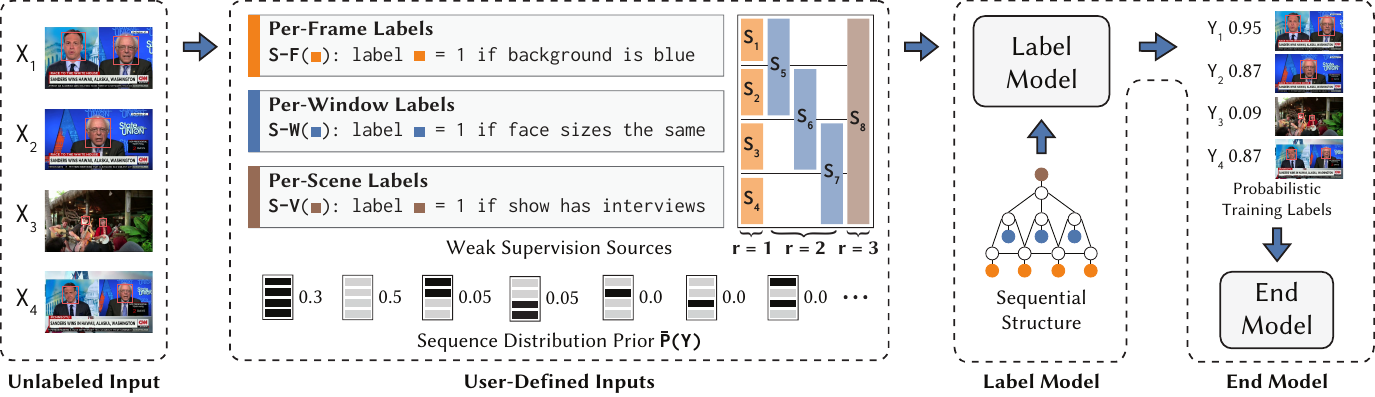}
 \caption{A schematic of the \sn{} pipeline.
 Users provide a set of unlabeled sequences $X = [X_1, \ldots, X_T]$; a set of weak supervision sources $S_1, \ldots, S_m$, each of which assigns labels at multiple resolutions (frame, window, scene); and a distribution prior $\bar{P}_Y$. The label model estimates the unknown accuracies and correlation strengths of the supervision sources and assigns probabilistic training labels to each element, which can be used to train a downstream end model.}
 \label{fig:overview}
\end{figure}

\section{Training Machine Learning Models with Weak Supervision}
\label{sec:para}
Practitioners often weakly supervise machine learning models by programmatically generating training labels through the process shown in Figure~\ref{fig:overview}. First, users provide multiple \emph{weak supervision sources}, which assign noisy labels to unlabeled data. These labels overlap and conflict, and a \emph{label model} is used to integrate them into probabilistic labels. These probabilistic labels are then used to train a discriminative model, which we refer to as an \emph{end model}.

While generating training labels across various sequential applications, we found that supervision sources often assign labels at \emph{different resolutions}: given a sequence with $T$ elements, sources can assign a single label per element, per collection of elements, or for the entire sequence. We describe a set of such supervision sources as \emph{multi-resolution}. For example in Figure~\ref{fig:multigranular}, to train an end model that detects interviews in TV shows, noisy labels can be assigned to each frame, each window, or each scene. Sources $\texttt{S-F}$, $\texttt{S-W}$, and $\texttt{S-V}$ each assign labels to a frame at resolution level $r=1$, a window at $r=2$, and scene at $r=3$, respectively. While each source operates at a specific resolution, the sources together are multi-resolution. 
The main challenge is combining source labels into probabilistic training labels by estimating source accuracies and correlations without ground-truth labels.

\subsection{Problem Setup}
\label{subsec:setup}
We set up our classification problem as follows:
\begin{itemize}
\item
Let $X = [X_1, X_2, \ldots, X_T] \in \mathcal{X}$ be an unlabeled sequence with $T$ elements (video frames in Figure~\ref{fig:multigranular}).
\item
For each sequence $X$, we assign labels to tasks at multiple resolutions ($Y_1$, $Y_{1,2}$, $Y_{seq}$ etc. in Figure~\ref{fig:multigranular}). We formally refer to the tasks using indices $\mathcal{T}=\{1,\ldots,|\mathcal{T}|\}$ ($|\mathcal{T}| = 4 + 3 + 1 = 8$ in Figure~\ref{fig:multigranular}).
\item
These tasks are at multiple resolutions ($3$ resolutions in Figure~\ref{fig:multigranular}) with the set of tasks at resolution $r$ denoted $R_r \subseteq \mathcal{T}$.
\item
$Y\in \mathcal{Y}$ is a vector $[y_1, \ldots, y_{|\mathcal{T}|}]$ of unobserved true labels for each task, and $(X,Y)$ are drawn i.i.d. from some distribution $\mathcal{D}$.
\end{itemize}

Users provide $m$ multi-resolution sources $S_1, \ldots, S_m$. Each source $S_j$ assign labels $\lambda_j$ to a set of tasks $\tau_j\subseteq\mathcal{T}$, (henceforth \emph{coverage set}), with size $s_j = |\tau_j|$. Each source only assigns labels at a specific resolution $r$, enforcing $\tau_j\subseteq R_r$ for fixed $r$. 
Users also provide a \emph{task dependency graph} $G_{\text{task}}$ specifying relations among tasks, a \emph{source dependency graph} $G_{\text{source}}$ specifying relations among supervision sources that arise due to shared inputs (Figure~\ref{fig:multigranular}), and a \emph{distribution prior} $\bar{P}(Y)$ describing likelihood of labels in a sequence (Figure~\ref{fig:overview}). While $G_{\text{source}}$ is user-defined, it can also be learned directly from the source outputs~\cite{bach2017learning, Varma19}. 

We want to apply weak supervision sources $S$ to an unlabeled dataset $X$ consisting of $n$ sequences, combine them into probabilistic labels, and use those to supervise an end model $f_w: \mathcal{X} \rightarrow \mathcal{Y}$ (Figure~\ref{fig:overview}). Since the labels from the supervision sources overlap and conflict, we learn a label model $P(Y | \lambda)$ that takes as input the noisy labels and outputs probabilistic labels \emph{at the required resolution} for the end model.

\subsection{Label Model}

\label{subsec:model}
Given inputs $X, S, G_{\text{task}}, G_{\text{source}}, \bar{P}(Y)$, we estimate the sources' unknown accuracies and correlation strengths.
Accuracy parameters $\mu$ and correlation parameters $\phi$ define a \emph{label model} $P_{\mu,\phi}(Y | \lambda)$, which can generate probabilistic training labels.
To recover parameters without ground-truth labels $Y$, we observe the agreements and disagreements of these noisy sources across different resolutions.

To recover these parameters, we form a graph $G$ describing all relations among sources and task labels, combining $G_{\text{source}}$ with $G_{\text{task}}$. The resulting graphical model encodes conditional independence structures. Specifically, if $(\lambda_j, \lambda_k)$ is not an edge in $G$, then $\lambda_j$ and $\lambda_k$ are independent conditioned on all other variables.

For ease of exposition, we assume the binary classification setting where $y_i \in \{-1,1\}$, $\lambda_i \in \{-1,0,1\}$ (we reserve $0$ for abstentions) for $T$ per-element tasks and $1$ per-sequence task. The accuracy parameter for source $j$ for some $Z, W \in \{-1,1\}^{s_j+1}$ is
\begin{align}
    \mu_j(Z, W) = P \left( \lambda_j = Z  \text{ }|\text{ } \ Y_{\tau_j} = W \right).
\end{align}
Intuitively, this parameter captures the \emph{accuracy} of each supervision source with respect to the ground truth labels in coverage set $\tau_j$. Next, for each correlation pair of sources $(\lambda_j,\lambda_k)$ and for some $Z_1 \in \{-1,1\}^{s_j}, Z_2 \in \{-1,1\}^{s_k}, W \in \{-1,1\}^{|\tau_j \cup \tau_k|}$, we wish to learn 
\begin{align}
    \phi_{j,k}(Z_1, Z_2, W) =  P \left( \lambda_j = Z_1, \lambda_k = Z_2  \text{ }|\text{ } \ Y_{\tau} = W \right),
\end{align}
where $\tau = \tau_j \cup \tau_k$. We can also learn the probability that a source abstains: details in the Appendix. 

\subsection{Parameter Reduction}
\label{subsec:tying}
Our assumption above of conditioning only on ground-truth labels for tasks in the source's coverage set instead of the full $\mathcal{T}$ greatly reduces the number of parameters. While we have at least $2^{T}$ parameters without the assumption, we now only need to learn $2^{2s_j}$ parameters per source, where $s_j$ tends to be much smaller than $T$.

In addition, we can model each source accuracy conditioned on each task, rather than over its full coverage set, reducing from $2^{2s_j}$ to $4s_j$ parameters and going from exponential to linear dependence on coverage set size, which is at most $T$.  
Lastly, we can also use parameter sharing: we share across sources that apply the same logic to label different, same-resolution tasks ($\mu_1 = \mu_2=\mu_3 = \mu_4$ in Figure~\ref{fig:multigranular}).

\section{Modeling Sequential Weak Supervision}
\label{sec:model}
The key challenge in sequential weak supervision settings is recovering the unknown accuracies and correlation strengths in our graphical model of multi-resolution sources, given the noisy labels, the source dependency structure, coverage sets, and distribution prior. 
We propose a provable algorithm that recovers the unique parameters with convergence guarantees by
reducing parameter recovery into systems of linear equations. These systems recover probability terms that involve the unobserved true label $Y$ by exploiting the pattern of agreement and disagreement among the noisy supervision sources at different levels of resolution (Section~\ref{sec:source_accuracy_estimation}). 
We theoretically analyze this algorithm, showing how the estimation error scales with the number of samples $n$, the number of sources $m$, and the length of the sequence $T$. Our approach additionally leverages repeated structures in sequential data by sharing appropriate parameters, significantly reducing sample complexity to no more than linear in the sequence length (Section~\ref{sec:scaling_sequential_supervision}).
Finally, we consider the impact of our estimation error on the end model trained with labels produced from our label model, showing that end model generalization scales with unlabeled data points as $O(1/\sqrt{n})$, the same asymptotic rate as if we had access to labeled data (Section~\ref{sec:scaling_sequential_supervision}).

\subsection{Source Accuracy Estimation Algorithm}
\label{sec:source_accuracy_estimation}
Our approach is shown in Algorithm~\ref{alg:label_model}: it takes as input samples of sources $\lambda_1, \ldots, \lambda_m$, the conditional independencies encoded in the graph $G$, and the prior $\bar{P}_{Y}$ and outputs the estimated accuracy and correlation parameters, $\hat{\mu}$ and $\hat{\phi}$ (for simplicity, we only show the steps for $\mu$ in  Algorithm~\ref{alg:label_model}.)

While we have access to the noisy labels assigned by the supervision sources, we do not observe the true labels $Y$ and therefore cannot calculate $\mu$ directly. However, given access to the user-defined distribution prior and the \emph{joint probabilities}, such as $P(\lambda_j(\{1\}), y_2)$, we can apply Bayes' law to estimate $\mu$ (Section 3.1.4). Since the joint probabilities also include the unobservable $Y$ term, we break it into the sum of \emph{product variables}, such as $P(\lambda_j(\{1\}) y_2 = 1)$ (Section 3.1.3). Note that we still have a dependence on the true label $Y$: to address this issue, we take advantage of (1) the conditional independence of some sources (Section 3.1.2), (2) the fact that we can observe the agreement and disagreements among the sources (Section 3.1.1), and (3) in the binary setting, $y^2 = 1$. 

We describe the steps of our algorithm and explain the assumptions we require, which involve the number of conditionally independent pairs of sources we have access to and how accurately they vote on their tasks.

\begin{algorithm}
	\SetKwRepeat{Do}{do}{for}%
	\SetKwInput{Input}{Input}
	\SetKwInOut{Output}{Output}
	\Input{Samples of sources $\lambda_1, \ldots, \lambda_n$, Dependency structure $G$, Dist. prior $\bar{P}(Y)$}
	\For {\text{source} $j \in \{1, \ldots, m\}$}{
	\For {coverage subsets $U, V \subseteq \tau_j$} { 
	         Using $G$, get source set $S_j$ where $\forall k,\ell \in S_j$, $\exists U_k, U_{\ell}$ \nonl \\ \quad s.t. $a_j(U,V)  \perp a_k(U_k,V) $,  $a_j(U,V)  \perp a_{\ell}(U_{\ell},V) $, $a_k(U_k,V)  \perp a_{\ell}(U_{\ell},V) $. Set $U_j = U$ \\
	\For {$k,\ell \in S_j \cup \{j\}$}{ 
		\vspace{0.2em}
		Calculate \textbf{gen. agreement measure}: $a_k(U_k,V) a_{\ell}(U_{\ell},V) = \prod_{U_k,U_{\ell}} \lambda_k(U_k) \lambda_{\ell}(U_{\ell})$ \\
		Form $q = \log \E{} {a_k(U_k,V) a_{\ell}(U_{\ell},V)}^2$ over coverage subsets $U_k,U_{\ell},V$ \\
	}

		Solve {\bf agreement-to-products} system: find $\ell_{U,V}$ s.t. $M\ell_{U,V} =  q$ \\	
	}
	\vspace{0.25em}
	Form product probability vector $r(\ell_{U,V})$ \\ 
	Solve {\bf products-to-joints} system: find $e$ s.t. $B_{2s_j}e = r$\\
	$\mu_j \leftarrow e / \bar{P}(Y)$ \\ 
}
		\Output{Parameter $\hat{\mu}$}
		\caption{Accuracy Parameter Estimation}
		\label{alg:label_model}
\end{algorithm}

\subsubsection{Generalized Agreement Measure}
Given the noisy labels assigned by the supervision sources, $\lambda_1, \ldots, \lambda_m$, we want some measure of agreement between these sources and the true label $Y$. For sources $j$ and $k$, let $U, U',V$ be subvectors of the coverage sets $\tau_j, \tau_k, \tau_j \cup \tau_k$, respectively. We use the notation $\prod_X A(X)$ to represent the product of all components of $A$ indexed by $X$. We then define a \emph{generalized agreement measure} as
$a_{j}(U, V) = \prod \lambda_{j}(U) \prod Y(V),$
which represents the agreement between the supervision source and the unknown true label when $U = V$ and $|U|=1$. Note that this term is \emph{not directly observable as it is a function of $Y$}.

Instead, we look at the product of two such terms:
$$a_{j}(U,V) a_{k}(U',V) = \prod_{U,U'} \lambda_{j}(U) \lambda_{k}(U') \prod_V (Y(V))^2 = \prod_{U,U'} \lambda_{j}(U) \lambda_{k}(U').$$
Since the $(Y(V))^2$ components multiply to 1 in the binary setting, we are able to represent the product of two generalized agreement measures in terms of the \emph{observable agreement and disagreement between supervision sources}. Therefore, we are able to calculate $a_{j}(U,V) a_{k}(U',V)$ across values of $U,V$ directly from the observed variables. 

\subsubsection{Agreements-to-Products System} 
Given the product of generalized agreement measures, we solve for terms that involve the true label $Y$, such as $a_{j}(U, V)$. Since we cannot observe these terms directly, we instead solve a system of equations that involve $\log ~ \E{} {a_j(U,V)}$, the \emph{log of the expectation of these values} when we have certain \emph{assumptions about independence} of different sources, conditioned on variables from $Y$. We give more details in the Appendix. As an example, note that if $\lambda_j(U)$ is independent of $\lambda_k(U')$ given $\prod Y(V)$ for $|V| = 1$, which is information that can be read off of the graphical model $G$, then
\begin{align}
\E{} { a_j(U,V)} \E{} {a_k(U',V)} = \E{} { a_j(U,V) a_k(U',V)} = \mathbb{E} \Big [\prod_{U,U'} \lambda_j(U) \lambda_k(U') \Big ]
\label{eq:indeps}
\end{align}
In other words, the conditional independencies of the sources translate to independencies of the accuracy-like terms $a$. 

Note that the middle term in \eqref{eq:indeps} can be calculated directly using observed $\lambda$'s. Now we wish to form a system of equations to solve for the terms on the left-most side.
We can take the log of the left-most term and the right-most term to form a system of linear equations, $M \ell = q$.
$M$ contains a row for each pair of sources, $\ell$ is the vector we want to solve for and contains the terms with $a_j(U,V)$, and $q$ is the vector we observe and contains the terms with $\lambda_j(U) \lambda_{k}(U')$.
We can solve this system up to sign if $M$ is full rank, which is true if $M$ has at least three rows. This is true if we have a group of at least three conditionally independent sources. 

\paragraph{Assumptions} We now have the notation to formally state our assumptions. We assume that each $a_j(U,V)$ has at least two other independent accuracies (equivalently, sources independent conditioned on $Y_V$) and $|\E{}{a_j(U,V)}| > 0$, i.e., our accuracies are correlated with our labels, positively or negatively), and that we have a list of such independencies (to see how to obtain such a list from the user-provided graphs, more information is in the Appendix). We also assume that on average, a group of connected sources have a better than random chance of agreeing with the labels, which enables us to recover the signs of the accuracies. These are standard weak supervision assumptions \cite{Ratner19}. 

Once we solve for $\E{}{ a_j(U,V)}$, we can calculate the \emph{product variable} probabilities $\rho_{j}(U,V) = P(a_j(U,V) = 1) = 1/2(1 + \E{}{ a_j(U,V)}).$
Note that product variable probabilities $\rho$ relies on the the true label $Y$, since $a_j(U,V)$ represents the generalized agreement between the source label and true label. However, we have now solved for this term \emph{despite not observing $Y$ directly}.

\subsubsection{Products-to-Joints System}
Given the product variable probabilities, we now want to solve for the \emph{joint probabilities} $p$ , such as $P(\lambda_{j,1}, Y_2)$. Fortunately, linear combinations of the appropriate $p_{j}(Z,W)  = P ( \lambda_j = Z , Y_{\tau_j} = W)$ result in $\rho_j(U,V)$ terms. Our goal is to solve for the unknown joint probabilities given the estimated $\rho_j$ product variables, user-defined distribution prior $\bar{P}_Y$, and observed labels from the sources $\lambda$. 

Say that $\lambda_1$ has coverage $\tau_1 = [1]$, so that it only votes on the value of $y_1$. Then, for $U = \{1\}, V = \{1\}$, we know $\rho_{1}(U,V) = P(\lambda_{1,1} y_1 = 1)$. But we have that $P(\lambda_{1,1} y_1 = 1) = p_{1}(+1,+1) + p_{1}(-1,-1)$, which is the agreement probability. Using similar logic, we can set up a series of linear equations:
\begin{align*}
\begin{bmatrix} 
1 & 1 & 1 & 1 \\
1 & 0 & 1 & 0 \\
1 & 1 & 0 & 0 \\
1 & 0 & 0 &1 
\end{bmatrix} \begin{bmatrix} p_{1}(+1, +1)\\ p_{1}(-1, +1)  \\ p_{1}(+1, -1) \\ p_{1}(-1, -1) \end{bmatrix} = \begin{bmatrix} 1 \\  P(\lambda_{1,1} = 1) \\ P(Y_1 = 1)  \\ \rho_{1}(U,V) \end{bmatrix} .
\end{align*}

Note that because of how we set up this system, the vector on the left-hand side contains the probabilities we need to estimate the joint probabilities. The right hand side vector contains either observable ($P(\lambda_{1,1} = 1)$), estimated ($\rho_{1}(U,V)$), or user-defined ($P(Y_1 = 1)$, from $\bar{P}_Y$) terms. In this example, the matrix is full-rank and we can therefore solve for the $p_1$ terms. 

To extend this system to the general case, we form a system of linear equations, $B_{2s_j} e = r$.
$B_{2s_j}$ is the \emph{products-to-joints} matrix (we discuss its form below), $e$ is the vector we want to solve for and contains the $p_{j}(Z,W)$ terms, and $r$ is the vector we have access to and contains observable, user-defined, and estimated $\rho_j(U,V)$ terms. 
$B_{2s_j}$ is $2^{2s_j} \times 2^{2s_j}$-dimensional 0/1 matrix. Let $\otimes$ be the Kronecker product; then, we can represent $B_{2s_j}$ as a Hadamard-like matrix (we show it is full rank in the Appendix):
\[B_{2s_j} =  \frac{1}{2} \begin{bmatrix} 1 & 1 \\ 1 & -1 \end{bmatrix} \otimes^{2s_j} + \frac{1}{2} 11^T.\]
We can now solve for terms required to calculate the joint probabilities and use them to obtain the $\mu$ parameters by using Bayes' law and the user-defined distribution prior 
$\mu_{j}(z,w) = p_{j}(Z,W) / P(Y_{\tau_j} = W).$ We can calculate the $\phi$ parameters in a similar fashion as $\mu$, except now we operate over \emph{pairs} of supervision sources, always working with products of correlated sources $\lambda_i \lambda_j$ (details in Appendix). 

\subsubsection{SGD-Based Variant}
Note that Algorithm~\ref{alg:label_model} explicitly builds and solves the linear systems that are set up via the agreement measure constraints. This involves a small amount of bookkeeping. However, there is a simple variant that relies on SGD for optimization and simply uses the constraints between the accuracies and correlations. That is, we use $\ell_2$ losses on the constraints (and additional ones required to make the probabilities consistent) and directly optimize over the accuracy and correlation variables $\mu, \phi$. Under the assumptions we have set up in this section, these algorithms are effectively equivalent; in the experiments, we use the SGD-based variant due to its ease of implementation in PyTorch.

\subsection{Theoretical Analysis: Scaling with Sequential Supervision}
\label{sec:theory}
\label{sec:scaling_sequential_supervision}
Our ultimate goal is to train an end model using the labels aggregated from the supervision sources using the estimated $\mu$ and $\phi$ for the label model. We first analyze Algorithm~\ref{alg:label_model} with parameter sharing as described in Section~\ref{subsec:tying} and discuss the general case in the Appendix. We bound our estimation error and observe the scaling in terms of the number of unlabeled samples $n$, the number of sources $m$, and the length of our sequence $T$. 
We then connect the generalization error to the end model to the estimation error of Algorithm~\ref{alg:label_model}, showing that generalization error scales asymptotically in $O(\sqrt{1/n})$, the same rate as supervised methods but in terms of number of \emph{unlabeled} sequences. 

We have $n$ samples of each of the $m$ sources for sequences of length $T$, and the graph structure $G = (V,E)$. We allow for coverage sets of size up to $T$. We assume the previously-stated conditions on the availability of conditionally independent sources are met, that $\forall j, | \mathbb{E} [ a_j(U,V) ] | \geq b^*_{\text{min}} > 0$, and that sign recovery is possible (for example, it is sufficient to have $ \forall j,U,V$,  $\sum_{\lambda_k \in S_j} \mathbb{E} [ a_k(U,V) ] > 0$ where $S_j$ is defined as in Algorithm~\ref{alg:label_model}). We also take $p_{\min}$ to be the smallest of the entries in $\bar{P}(Y)$. Let $\| \cdot \|$ be the spectral norm.

\begin{restatable}{theorem}{thmmuest}
Under the assumptions above, let $\hat{\mu}$ and $\hat{\phi}$ be estimates of the true  $\mu^*$ and $\phi^*$ produced with Algorithm~\ref{alg:label_model} with parameter reduction. Then,
\begin{align}
\E{}{ \| \hat{\mu} - \mu^* \| } \leq \sqrt{mT} \frac{24}{p_{\min} b^*_{\min}} \|B_{2T}^{-1}\|  \|M^\dagger\| \left( \sqrt{\frac{18 \log(12)}{n}} + \frac{2\log(12)}{n}\right). 
\label{eq:bd1}
\end{align}
The expectation $\mathbb{E} [ \| \hat{\phi} - \phi^* \| ]$ satisfies the bound \eqref{eq:bd1}, replacing $\sqrt{mT}$ with $mT$ and $B_2$ with $B_4$. 
\label{thm:sample_comp}
\end{restatable}

\paragraph{Interpreting the Theorem} The above formula scales with $n$ as $O(\sqrt{1/n})$, and critically, \emph{no more than linear in $T$}. We prove a more general bound without parameter reduction, which scales exponentially in $T$ in Appendix~\ref{sec:appendix}. A matching lower bound in $T$ requires $\Omega(2^{\frac{3}{2}T})$ samples. The expression scales with $m$ as $O(\sqrt{m})$ and $O(m)$ for estimating $\mu$ and $\phi$, respectively. The standard scaling factors for the random vectors produced by the sources are $m$ and $m^2$; however, we need \emph{only two additional sources for each source}, leading to the $\sqrt{m}$ and $m$ rates. The linear systems enter the expression only via $\|B^\dagger\|$. These are fixed; in particular, $\|B_2^\dagger\| = 1.366$ and $\|B_4^\dagger\| = 1.112$.

\paragraph{End Model Generalization}
After obtaining the label model parameters, we use them to generate probabilistic training labels for the resolution required by the end model.
The parameter error bounds from Theorem~\ref{thm:sample_comp} allow us to apply a result from \cite{Ratner19},
which states that under the common weak supervision assumptions (e.g., the parameters of the distribution we seek to learn are in the space of the true distribution), 
the generalization error for $Y$ satisfies $\E{}{ l(\hat{w},X, Y) -  l(w^*, X, Y) } \leq \gamma + 8 (\| \hat{\mu} - \mu^*\| + \|\hat{\phi} - \phi^*\|)$.
Here, $l$ is a bounded loss function and $w$ are the parameters of an end model $f_w : \mathcal{X} \rightarrow \mathcal{Y}$.
We also have $\hat{w}$ as the parameters learned with the estimated label model using $\mu$ and $\phi$, and 
$w^* = \text{argmin}_w l(w,X,Y)$, the minimum in the supervised case.
This result states that the generalization error for our end models \emph{scales with the amount of unlabeled data as $O(1/\sqrt{n})$}, the same asymptotic rate as if we had access to the true labels.

\section{Experimental Results}
\label{sec:exp}

We validate \sn{} on real-world sequential classification problems, comparing end model performance trained on labels from \sn{} and other baselines. 
\sn{} improves over traditional supervision and other state-of-the-art weak supervision methods by \tradimp{} and \dpimp{} F1 points on average in terms of end model performance, respectively. 
We also conduct ablations to compare parameter reduction techniques, the effect of modeling dependencies, and advantage of using a user-defined prior, with average improvements of \shareimp{}, \depsimp{}, and \priorimp{} F1 points, respectively. Finally, we show how our model scales with the amount of unlabeled data, coming within $0.1$ F1 points of a model trained on $686\times$ more ground-truth labels.

\begin{figure}
  \includegraphics[width=\columnwidth]{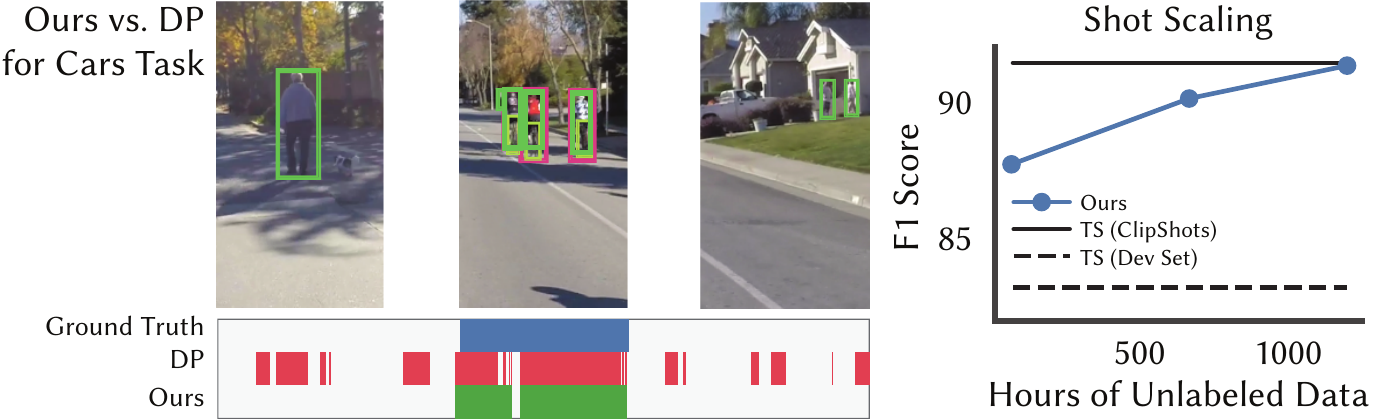}
    \caption{(Left) \sn{} has fewer false positives than data programming since it uses sequential correlations and distributional knowledge to assign better training labels.
    (Right) Increasing unlabeled data can help match a benchmark model trained with $686\times$ more ground truth data (Shot) and increase precision as needed for clinicians (BAV).}
  \label{fig:cars_scale_up}
\end{figure}

\subsection{Datasets}
We consider two types of tasks, spanning various modalities: (a) tasks that are expensive and slow to label due to the domain expertise required, and (b) previously studied, large-scale tasks with strong baselines often based on hand-labeled data developed over months. All datasets include a small hand-labeled \emph{development set} ($<10\%$ of the unlabeled data) used to tune supervision sources and end model hyperparameters. Results are reported on test set as the mean $\pm$ S.D. of F1 scores across 5 random weight initializations. See Appendix~\ref{sec:appendix} for additional task and dataset details, precision and recall scores, and end model architectures.

\paragraph*{Domain Expertise}
These tasks can require hours of expensive expert annotations to build large-scale training sets.
Bicuspid Aortic Valve \textbf{(BAV)}~\cite{fries2018weakly} is classifying a congenital heart defect over MRI videos from a population-scale dataset~\cite{sudlow2015uk}. Labels generated from \sn{} and sources based on characteristics like heart area and perimeter are \emph{validated by cardiologists}.
Interview Detection \textbf{(Interview)} identifies interviews of Bernie Sanders from TV news broadcasts; across a large corpus of TV news, interviews with Sanders are rare, so it requires significant labeling effort to curate a training set.
Freezing Gait \textbf{(Gait)} is ankle sensor data from Parkinson's patients and the task is to classify abnormal gait, using supervision sources over characteristics like peak-to-peak distance.
Finally, \textbf{EHR} consists of tagging mentions of disorders in patient notes from electronic health records. 
We only report label model results for \textbf{EHR}, but \sn{} improves over a majority vote baseline by 3.7 F1 points (Appendix~\ref{sec:appendix_results}).

\paragraph*{Large-Scale}
Movie Shot Detection \textbf{(Movie)} classifies frames that contain a change in scene using sources that use information about pixel values, frame-level metadata, and sequence-level changes. This task is well-studied in literature~\cite{tang2018dsm,hassanien2017deepsbd} but adapting the method to specialized videos requires manually labeling thousands of minutes of video. Instead, we use \emph{$686\times$} fewer ground truth labels and various supervision sources to \emph{match the performance of a model pre-trained on a benchmark dataset with ground truth labels} (Figure~\ref{fig:cars_scale_up}). \textbf{Basketball} operates over a subset of ActivityNet~\cite{caba2015activitynet} and uses supervision sources over frames and sequences. Finally, we use a representative dataset from a large automated driving company \textbf{(Cars)}~\cite{cydet} and show that \emph{we outperform their best baseline by $9.9$ F1 points}. 
The \textbf{Cars} end model is proprietary, so we only report label model
results (Appendix~\ref{sec:appendix_results}).


\begin{table}[htbp]
    \centering
    \small
    \begin{tabular}{@{}lccccccccc@{}}
        \toprule
        \multicolumn{1}{c}{} & \textbf{}     & \textbf{}  & \multicolumn{4}{c}{\textbf{End Model Performance}}                        & \multicolumn{3}{c}{\textbf{Improvement}} \\ \cmidrule(l){4-7} \cmidrule(l){8-10} 
        \textbf{Task}     & \textbf{prop} & \textbf{T} & \textbf{TS}     & \textbf{MV}    & \textbf{DP}    & \textbf{\sn}   & \textbf{TS}   & \textbf{MV}   & \textbf{DP}   \\ \midrule
        BAV                  & 0.07          & 5          & 22.1 $\pm$ 5.1  & 6.2 $\pm$ 7.6  & 53.2 $\pm$ 4.4 & \textbf{53.8 $\pm$ 7.6} & +31.7         & +47.6         & +0.6         \\
        Interview            & 0.03          & 5          & 80.0 $\pm$ 3.4  & 58.0 $\pm$ 5.3 &  8.7 $\pm$ 0.2 & \textbf{92.0 $\pm$ 2.2} & +12.0         & +34.0        & +83.3          \\
        Gait                 & 0.33          & 5          & 47.5 $\pm$ 14.9 & 61.6 $\pm$ 0.4 & 62.9 $\pm$ 0.7 & \textbf{68.0 $\pm$ 0.7} & +20.5         & +6.4          & +5.1          \\
        Shot                 & 0.10          & 5          & 83.2 $\pm$ 1.0  & 86.0 $\pm$ 0.9 & 86.2 $\pm$ 1.1 & \textbf{87.7 $\pm$ 1.0} & +4.5          & +1.7          & +1.5          \\
        Basketball           & 0.12          & 5          & 26.8 $\pm$ 1.3  & 8.1 $\pm$ 5.4  & 7.7 $\pm$ 3.3  & \textbf{38.2 $\pm$ 4.1} & +11.4         & +30.1         & +30.5        \\
        \bottomrule \\
        \vspace{-2em}
        \end{tabular}
        \caption{End model performance in terms of F1 score (mean $\pm$ std.dev). Improvement in terms of mean F1 score. prop: proportion of positive examples in the dev set, $T$: number of elements in a sequence. We compare end model performance on labels from labeled dev set (TS), majority vote across sources (MV), and data programming (DP) and outperform each across all tasks.}
    \label{table:baseline}
\end{table}

\subsection{Baselines}
For the tasks described above, we compare to the following baselines (Table~\ref{table:baseline}): \textit{Traditional Supervision (TS)} in which end models are trained using the hand-labeled development set; \textit{Non-sequential Majority Vote (MV):} in which we force all supervision sources assign labels per-element, and calculate training labels by taking majority vote across sources; and \textit{Data Programming (DP)~\cite{Ratner16}:} a state-of-the-art weak supervision technique that learns the accuracies of the sources but does not model sequential correlations.

In tasks with domain expertise required, our approach improves over traditional supervision by up to $36.8$ F1 points and continually improves precision as we add unlabeled data, as shown in Figure~\ref{fig:cars_scale_up}.
Large-scale datasets have manually curated baselines developed over \emph{months}; \sn{} is still able to improve over baselines by up to $30.5$ F1 points by capturing sequential relations properly --- as shown in Figure~\ref{fig:cars_scale_up}, only modeling source accuracies (DP) can fail to take into account the distribution prior and sequential correlations among sources that can help filter false positives, which \sn{} does successfully. 

\subsection{Ablations}
We demonstrate how each component of our model is critical by comparing end model performance trained on labels from \sn{} without any sequential dependencies, \sn{} without parameter sharing for sources with shared logic (Section~\ref{subsec:tying}), and \sn{} with various distribution priors: user-defined, development-set based, and uniform. We report these comparisons in Appendix~\ref{sec:appendix} and summarize results here. 

Without sequential dependencies, end model performance worsens by \depsimp{} F1 points on average, highlighting the importance of modeling correlations among sources. We see that sharing parameters among sources that use the same logic to assign labels at the same resolution performs \shareimp{} F1 points better on average.
Using a user-defined distribution prior improves over using a uniform distribution prior by \priorimp{} F1 points and a development-set based distribution prior by \priorimpdev{} F1 points on average, highlighting how domain knowledge in forms other than supervision sources is key to generating high quality training labels. 

\section{Related Work}
\label{sec:rel}

Our work is related to several weak supervision techniques such as traditional distant supervision~\cite{mintz2009distant,craven1999constructing,hoffmann2011knowledge,takamatsu2012reducing}, co-training methods~\cite{blum1998combining}, pattern-based supervision~\cite{gupta2014improved} and feature annotation techniques~\cite{mann2010generalized,zaidan2008modeling,liang2009learning}. Recent works also use generative models~\cite{ratner2016data,Ratner18,bach2017learning} and other methods~\cite{guan2018said,khetan2017learning} to integrate these noisy sources. However, these approaches do not handle sequential correlations or multi-resolution sources and require expensive sampling-based techniques that can lead to non-identifiability. One recent approach directly models weak supervision sources using deep generative models for trajectory data, but does not use weak supervision sources to label training data for arbitrary end models~\cite{zhan2019sequentialws}. Our work is also related to recent techniques for estimating classifier accuracies without labeled data in the presence of structural constraints~\cite{platanios2017estimating}.
Our work is also related to crowdsourcing~\cite{karger2011iterative,dawid1979maximum}, specifically to spectral and method of moments-based approaches~\cite{zhang2016spectral,dalvi2013aggregating,ghosh2011moderates,anandkumar2014tensor}. Our work focuses on the settings that is not covered by crowdsourcing, such as multi-resolution sources, sequential correlation structures, and regimes in which a small number of labelers, or sources, assign noisy labels to a large set of datapoints. We also theoretically characterize how the end model trained on labels from noisy sources generalizes. 


\section{Conclusion}
\label{sec:conc}

We propose \sn{}, the first weak supervision framework that integrates multi-resolution weak supervision sources including complex dependency structures to assign probabilistic labels to training sets without using any hand-labeled data. We prove that our approach can uniquely recover the parameters associated with supervision sources under mild conditions, and that the sample complexity of an end model trained using noisy sources matches that of supervised approaches.
Experimentally, we demonstrate that \sn{} improves over traditional supervision by \tradimp{} F1 points and existing weak supervision approaches by \dpimp{} F1 points for real-world classification tasks training over large, population-scale biomedical repositories like UKBiobank~\cite{sudlow2015uk} and industry-scale video datasets for self-driving cars.

\section*{Acknowledgments}
\noindent

We gratefully acknowledge the support of DARPA under Nos. FA87501720095 (D3M), FA86501827865 (SDH), and FA86501827882 (ASED); NIH under No. U54EB020405 (Mobilize), NSF under Nos. CCF1763315 (Beyond Sparsity), CCF1563078 (Volume to Velocity), and 1937301 (RTML);  ONR under No. N000141712266 (Unifying Weak Supervision); the Moore Foundation, NXP, Xilinx, LETI-CEA, Intel, IBM, Microsoft, NEC, Toshiba, TSMC, ARM, Hitachi, BASF, Accenture, Ericsson, Qualcomm, Analog Devices, the Okawa Foundation, American Family Insurance, Google Cloud, Swiss Re, 
Brown Institute for Media Innovation,
the National Science Foundation (NSF) Graduate Research Fellowship under No. DGE-114747, Joseph W. and Hon Mai Goodman Stanford Graduate Fellowship,
Department of Defense (DoD) through the National Defense Science and
Engineering Graduate Fellowship (NDSEG) Program, 
and members of the Stanford DAWN project: Teradata, Facebook, Google, Ant Financial, NEC, VMWare, and Infosys. The U.S. Government is authorized to reproduce 
and distribute reprints for Governmental purposes notwithstanding any copyright notation thereon. Any opinions, findings, and conclusions or recommendations expressed in this material are those of the authors and do not necessarily reflect the views, policies, or endorsements, either expressed or implied, of DARPA, NIH, ONR, or the U.S. Government.

\bibliographystyle{abbrv}
\bibliography{time}

\newpage
\appendix

First we include a glossary of the terminology and notation used throughout this paper for ease of reference. Afterwards, we provide our theoretical analysis, and extended theorem statement, proofs, and more details on model identifiability. Lastly, we include additional details and experiments.

 \section{Glossary}
The glossary is given in Table~\ref{table:glossary} below.

\begin{table*}[h]
\centering
\begin{tabular}{l l}
\toprule
Symbol & Used for \\
\midrule
$\x$ & Unlabeled data sequence, $\x = [X_1, X_2, \ldots, X_T] \in \mathcal{X}$ \\
$T$ & Length of the unlabeled data sequence \\
$n$ & Number of data sequences \\
$\mathcal{T}$ & Task indices \\
$\y$ & Latent, ground-truth label vector, $\y = [y_1, y_2, \ldots, y_T, y_{T+1}, \ldots, y_{|\mathcal{T}|}] \in \mathcal{Y}$  \\
$y_i$ & Ground-truth label for $i$th task, $y_i \in \{-1,1\}$  \\
$\mathcal{D}$ & Distribution from which we assume $(\x, \y)$ data points are sampled i.i.d. \\
$r$ & Resolution level. $r=1$ refers to resolution level in which each of the $T$ elements is labeled\\
$R_r \subseteq \mathcal{T}$ & Set of task indices that are at resolution $r$\\
$G_{\text{task}}$ & Task dependency graph describing the correlation structure among tasks in a graph\\
$m$ & Number of sources \\
$\lf_i$ & Output of $S_j$ for $\x$, $\lf_i \in \{-1,1,0\}$ \\
$\tau_ j$ & Coverage set of $\lf_j$ - the task indices $\tau_j\subseteq\mathcal{T}$ for which $S_j$ can label. For $S_j$ operating at resolution $r$, \\
& $\tau_j\subseteq R_r$\\
$s_j$ & Size of the $j$th source coverage set, $s_j = |\tau_j|$ \\
$G_{\text{source}}$ & Source dependency graph that describe the correlation structure among source, \\
 & particularly for correlations due to shared inputs \\
$G$ & Full dependency graph, $G = (V,E)$ obtained by combining $G_{\text{source}}$ and $G_{\text{task}}$. \\
&$V=\{\lf_1, \ldots ,\lf_m\} \cup \{y_1, \ldots, y_{|\mathcal{T}|}\}$ \\
$\mu_j$ & Accuracy parameter for source $j$; $\mu_j(Z, W) = P \left( \lambda_j = Z  \text{ }|\text{ } \ Y_{\tau_j} = W \right)$ \\
$\phi_{j,k}$ & Correlation parameter for sources $j,k$; \\
& \qquad $\phi_{j,k}(Z_1, Z_2, W) =  P \left( \lambda_j = Z_1, \lambda_k = Z_2  \text{ }|\text{ } \ Y_{\tau} = W \right)$ \\
$\bar{P}_Y$ & Class prior for the $Y$ label vector \\
$a_{j}$ & Generalized agreement measure; $a_{j}(U, V) = \prod \lambda_{j}(U) \prod Y(V)$; \\
& \qquad Products are observable for common $V$ \\
$U,V$ & Subsets of the coverage set $\tau_j$ \\
$M$ & Matrix for first linear system, each row encodes pairs of agreements that factorize \\
$q$ & Observable vector with $\E{}{\lambda_j(U) \lambda_{k}(U')}$ terms \\
$\ell$ & Solution for products variable system system $M \ell = q$; \\
& \qquad Contains the terms $ \log \E{}{a_k(U',V)}^2$ \\
$\rho_{j}$ & Product variable obtainable from generalized agreement; \\ 
& \qquad $\rho_{j}(U,V) = P(a_j(U,V) = 1) = \frac{1}{2} + \frac{1}{2} \E{}{ a_j(U,V)}$ \\
$p_j$ & Joint distribution for source $j$ and $Y_{\tau_j}$ ;$p_{j}(Z,W)  = P ( \lambda_j = Z , Y_{\tau_j} = W)$ \\
$B_{2s_j}$ & Products-to-joints transformation matrix \\ 
$r$ &  Vector containing the $\rho_j(U,V)$; is estimated after products variable system is solved \\ 
$e$ &Vector containing the $p_{j}(Z,W)$, solution to products-to-joints system $B_{2s_j} e = r$  \\
\end{tabular}
\caption{
	Glossary of variables and symbols used in this paper.
}
\label{table:glossary}
\end{table*}

\section{Proofs and Extended Theoretical Analysis}

We give more details on the theoretical results we provided in the body. We start by providing the proof of Theorem~\ref{thm:sample_comp}. Afterwards, we discuss model identifiability, expressing tradeoffs involving multi-resolution models. Finally, we provide further detail on simulations and how to access conditional independencies from graphs.

First, we begin with a proof of Theorem~\ref{thm:sample_comp}. The following lemma will be useful. We use a little bit of notation. Let $D = (d_1, \ldots, d_t)$ be a random vector in $\{-1, +1\}^t$. For particular vectors $U,V,Z \in \{-1, +1\}^t$, we write $p_D(Z) = P(d_1 = z_1, \ldots, d_t = z_t)$ and $\rho_D(U) = P(\prod_U D_U = 1)$. The $p$ term is a joint probability and the $\rho$ term is a product probability. Let 
\[B_{t} =  \frac{1}{2} \begin{bmatrix} 1 & 1 \\ 1 & -1 \end{bmatrix} \otimes^{t} + \frac{1}{2} 11^T.\] 
Here, $1$ is the all $1$'s vector and $A \otimes^k$ represents taking the Kronecker product $A \otimes A$ a total of $k$ times.

Let the vector $e$ contain the $2^t$ entries $p_D(Z)$, with $Z$ taken in the following order. $z_t = +1$ for the first $2^{t-1}$ entries and $-1$ for the latter $2^{t-1}$ entries, $z_{t-1} = +1$ for the first $2^{t-2}$ entries, and so on, so that $z_1$ alternates between $+1$ and $-1$. Similarly, let the vector $r$ contain the $2^t$ choices of $\rho_D(U)$ running over all $2^t$ subsets of $\{1,\ldots, t\}$. We write $\rho_D(\emptyset) = 1$. Then, the ordering of the entry in $r$ is similar to the ones in $e$: the first half of the $U$ terms in $\rho_D$ do not contain the entry $t$, the latter half do, and so on, so that every alternating entry contains the entry $1$. Then,

\begin{lemma}
With the setup above, $B_{t} e = r$.
\label{lem:b}
\end{lemma}
\begin{proof}
We prove the result by induction on $t$. For the base case we take $t=1$. Then, using the above formula for $B_1$, we must have the following, which is clearly true 
\begin{align*}
\begin{bmatrix} 
1 & 1  \\
1 & 0 
\end{bmatrix} 
\begin{bmatrix} p_D([1]) \\ p_D([-1]) \end{bmatrix}= 
\begin{bmatrix} 
\rho_D(\emptyset) \\ \rho_D(\{1\})
\end{bmatrix}.
\end{align*}

Next, we assume the result holds for $t=k$, and we show it is true for $t=k+1$. That is, we have $B_t e_k = r_k$, and we'd like to show that $B_{t+1} e_{k+1} = r_{k+1}$. It follows from the definition of $B_{t+1}$ that it can be decomposed as 
\[B_{t+1}  = 
\begin{bmatrix} 
B_t & B_t  \\
B_t & \bar{B_t} 
\end{bmatrix} ,
\]
where the bar indicates flipped 1's and 0's. Furthermore, from our ordering, we have that $e_{k+1}$ can be written as $[e_k \cap (d_{t+1} = 1) , e_k \cap (d_{t+1} = -1)]^T$, where we augment each probability term in $e_k$ with either $d_{t+1} = 1$ or $d_{t+1} = -1$. Similarly, we have that $r_{k+1} = [r_k ; r_k \cup {d_{t+1}}]^T$. Then, what we want to show, $B_{t+1} e_{k+1} = r_{k+1}$, is equivalent to
\begin{align*}
\begin{bmatrix} 
B_t & B_t  \\
B_t & \bar{B_t} 
\end{bmatrix} 
\begin{bmatrix}e_k \cap (d_{t+1} = 1) \\ e_k \cap (d_{t+1} = -1)\end{bmatrix}= 
\begin{bmatrix} 
 r_k \\ r_k \cup {d_{t+1}}
\end{bmatrix}.
\end{align*}

The result follows almost immediately. For the block of $r_{k+1}$ on the top, we are summing the $B_t$ including both cases $d_{t+1} = 1$ and $d_{t+1} = -1$, which sums up to $r_k$ using the law of total probability and the inductive hypothesis. For the bottom block, we are summing over the probabilities of cases where $d_{t+1} = 1$ and the other terms in each $U$ multiply to 1, along with those with $d_{t+1} = -1$ and the others multiplying to $-1$, which gives all the cases where the terms in $U \cup \{t+1\}$ multiply to 1, which is indeed the lower subvector on the right. Thus we are done.
\end{proof}

Now we are ready for the proof of Theorem~\ref{thm:sample_comp}. Recall our assumptions: we see $n$ samples for each of the $m$ sources for sequences of length $T$, and we have the graph structure $G = (V,E)$. Our coverage sets are of length up to $T$. We have sufficiently many conditionally independent sources, and also that $\forall j, | \mathbb{E} [ a_j(U,V) ] | \geq b^*_{\text{min}} > 0$. Finally, we assume that sign recovery is possible. One way to have this is to require that $ \forall j,U,V$,  $\sum_{\lambda_k \in S_j} \mathbb{E} [ a_k(U,V) ] > 0$ where $S_j$ is defined as in Algorithm~\ref{alg:label_model}. 

With this, we prove a more general statement, without parameter reduction, and then we show how to obtain from it the parameter reduction case in Theorem~\ref{thm:sample_comp}.
\begin{restatable}{theorem}{thmext}
Under the assumptions previously described, let $\hat{\mu}$ and $\hat{\phi}$ be estimates of the true  $\mu^*$ and $\phi^*$ produced with Algorithm~\ref{alg:label_model}. Then,
\begin{align}
\E{}{ \| \hat{\mu} - \mu^* \| } \leq 2^{2T} \sqrt{m} \frac{6}{p_{\min}b^*_{\min}} \|B_{2T}^{-1}\|  \|M^\dagger\| \left( \sqrt{\frac{18 \log(6\times 2^T)}{n}} + \frac{2\log(6 \times 2^T)}{n}\right). 
\label{eq:bdt}
\end{align}
The expectation $\E{}{\norm{ \hat{\phi} - \phi^*}}$ satisfies the bound \eqref{eq:bd1}, replacing $\sqrt{m}$ with $m$, $B_{2T}$ with $B_{4T}$, and $2^T$ with $2^{2T}$. 

Moreover, $\|M^\dagger\| = 1$, and with parameter tying, ~\eqref{eq:bdt} reduces to the expression in Theorem~\ref{thm:sample_comp}.
\label{thm:sample_comp_big}
\end{restatable}

\begin{proof}
There are two steps to the proof. First, we must show that the true parameters $\mu^*$ and $\phi^*$ are produced by Algorithm~\ref{alg:label_model} when we have access to the true, population-level joint probabilities of the sources. Afterwards, we compute the noisy version due to sampling error.

\paragraph{Population-Level Result}
There are two necessary results: first, we need to show that the true parameters are solutions to the second system, and, secondly, that they are the unique solutions. We start with the first system. We work over each source $j$ and some fixed $U,V$ coverage sets for $\lambda_j$ and $Y$. From the algorithm, we have a set of sources $S_j \in V(G)$ with $c = |S_j \cup \{j\}| \geq 3$ so that $\lambda_j(U)$ and $\lambda_{k}(U_k)$ are independent conditioned on $\prod Y(V)$, and likewise for each pair of sources in $S_j$ over their corresponding $U$'s. For simplicity, we take $c=3$ exactly, but it is easy to solve larger systems, and the proof below does not depend on the value of $c$. We then say we have $S_j \cup \{j\} = \{j,k,f\}$, that is, our sources are $\lambda_j, \lambda_k, \lambda_f$.

We formulate the resulting matrix $M$. Recall that each row of $M$ corresponds to an equation 
\[\log \E{}{ a_j(U,V)}^2 + \log \E{}{a_k(U',V)}^2 = \log \E{} {\prod_U \lambda_j(U) \lambda_{k}(U) }^2.\]
We have 3 such equations, for the pairs $(j,k), (j,f)$, and $(k,f)$. Then, our linear system is $M\ell = q$, given by
\begin{align}
\begin{bmatrix} 
1 & 1 & 0 \\
1 & 0 & 1  \\
0 & 1 & 1  \\
\end{bmatrix} 
\begin{bmatrix} \log \E{}{ a_j(U,V)}^2 \\ \log \E{}{ a_k(U',V)}^2 \\ \log \E{}{ a_f(U'',V)}^2 \end{bmatrix} = \begin{bmatrix} 
\log \E{} {\prod_{U,U'} \lambda_j(U) \lambda_{k}(U') }^2 \\  
\log \E{} {\prod_{U,U''} \lambda_j(U) \lambda_{f}(U'') }^2\\ 
\log \E{} {\prod_{U',U''} \lambda_k(U') \lambda_{f}(U'') }^2  
\end{bmatrix}.
\label{eq:system3}
\end{align}

The $3 \times 3$ matrix above is full-rank, so that we can obtain the unique solution---the vector of $\E{}{a_j(U,V)}^2$ terms. Note that the above easily extends for more than $3$ such equations. In the case of $c > 3$ sources, the matrix has $\binom{c}{2}$ rows, each with exactly two $1$'s. The resulting matrix is also full-rank. To see this, we apply a result of A. M. Odlyzko \cite{Odl} that states for 0/1, constant row-sum matrices (sum is $2$ in our case), with $c>4$, $\binom{c-1}{2}+1$ distinct rows always guarantee that the matrix is full rank. For $c=4$, Odlyzko's result requires $2\binom{c-2}{(c-2)/2)} + 1 = 5$ distinct rows, and we have $\binom{4}{2} = 6$, so this case works as well.

The only remaining step for the first system is to deal with identifiability. The above allows us to get the squares of the $\E{}{a_j(U,V)}$ terms. We thus need to recover their signs by using the sign recovery assumption. One obvious approach is to require that each of our sources has accuracy that satisfies $\E{}{a_j(U,V)} > 0$. However, much milder assumptions are possible:
Note that once we know the sign of a single source accuracy, say $k$, we get all others, since for each source $f$, we have an equation for each pair $(k,f)$. So in fact, as mention in the assumptions, the much weaker requirement that $\sum_{\lambda_k \in S_j} \mathbb{E} [ a_k(U,V) ] > 0$ is sufficient. There are other potential variants as well.

Now, by running the above procedure for all sources $j$ and all the $U,V$'s, we are ready to form the second system. We apply Lemma~\ref{lem:b} with $t = 2s_j$ and $D = (\lambda_j, Y_{\tau_j})$ obtaining that $B_{2s_j} e = r$.

To solve uniquely, we need to show that $B_{2s_j}$ is also full-rank. Note that the rank is not affected by adding a constant to each entry, unless it produces a 0 row, which it does not in this case, since the Hadamard matrix here was selected to have no all $-1$ rows. The Kronecker product of matrices multiplies the corresponding ranks. Since $\begin{bmatrix} 1 & 1 \\ 1 & -1 \end{bmatrix}$ is full-rank, $B_{2s_j}$ must be as well. Thus, there is a unique solution to our system, and it is indeed the desired $p_{j}(Z,W)$'s. Moreover, we can uniquely recover the $\mu$ parameters as well, as long as we know the distribution of $Y$. Finally, the same logic applies to the $\phi$ parameters, concluding the argument for the population-level result.

\paragraph{Sampling-Level Result}
Now we apply a matrix concentration inequality to bound the sampling error. First, we require more notation. 

For the first system, we'd like to estimate terms like $\log \E{} {\prod_U \lambda_j(U) \lambda_{k}(U')}^2$. Again, say we are working with three sources $j,k,f$. Let us say they all work with the same coverage subset $U$ of maximal size $T$, which is an upper bound for our general case. 

Then, for $j$ we create $o = 2^{T}$ indicator variables, one for each configuration of $\lambda_j(U)$. 
Call these variables $c_{j,1}, c_{j,2}, \ldots, c_{j,o}$, and likewise for $k$ and $l$. For example, if $T = 2$, then $c_{j,1}, \ldots, c_{j,4}$ correspond to $\one \{\lambda_{j} = [-1, -1]\}, \one \{\lambda_{j} = [-1, 1]\}$, and so on.

We stack these vectors together to form the vector $c$ of length $3o$, and we estimate the matrix $O^* = \E{}{cc^T}$. We do this by estimating $c^1, c^2, \ldots, c^n$ from our samples $\lambda^1, \ldots, \lambda^n$, filling in the indicators accordingly. Then, we use the estimate $\hat{O} = \frac{1}{n} \sum_{i=1}^n c^i (c^i)^T$. Our first step is obtaining a bound on $\|\Delta_O\| = \|O^* - \hat{O}\|$. 
We do this by estimating $c^1, c^2, \ldots, c^n$ from our samples $\lambda^1, \ldots, \lambda^n$, filling in the indicators accordingly. Then, we use the estimate $\hat{O} = \frac{1}{n} \sum_{i=1}^n c^i$. Our first step is obtaining a bound on $\|\Delta_O\| = \|O^* - \hat{O}\|$.

We use the matrix Bernstein inequality following \cite{Tropp15}. 
Let $\Delta_O = \hat{O} - O^* = \sum_{i=1}^n S_i$, where $S_i = \frac{1}{n} (c^i (c^i)^T - O^*)$. 
Then, using Theorem 1.6.2 in \cite{Tropp15}, we can write
\begin{equation}
\E{}{\|\Delta_O\|} \leq \sqrt{2v(\Delta_O) \log(6o)} + \frac{1}{3} L \log(6o).
\label{eq:bernstein}
\end{equation}
Here, the dimensions of $\Delta_O$ are $3o \times 3o$, $v(\Delta_O)$ is the variance of $\Delta_O$, which is defined as $\|\E{}{\Delta_O\Delta_O^T}\|$, and, finally, $L$ is an upper bound on $\|\frac{1}{n} (c^i (c^i)^T - O^*)\|$. 
We can apply the result by taking the bound on $\|c^i\|$ to be $3o$ and a bound on $\|O^*\|$ to be $3o$ as well. 
We also need to bound the variance $v(\Delta_O)$; using the same ideas as in \cite{Tropp15}, we get $v(\Delta_O) \leq \frac{3o^2\|O\|}{n}$. Then, we have that

\begin{equation}
\label{eq:conc_bound}
\E{}{\|\Delta_O\|} \leq \sqrt{\frac{18 o^2  \log(6o)}{n}} + \frac{2o \log(6o)}{n}.
\end{equation}

This tells us how to bound the error between all the configurations that $\lambda_j$ and $\lambda_k$ can take on. 
We define $b^*$ as 
\[b^* = 
 \begin{bmatrix} 
\E{} {\prod_{U,U'} \lambda_j(U) \lambda_{k}(U') }\\  
\E{} {\prod_{U,U'} \lambda_j(U) \lambda_{f}(U') }\\ 
\E{} {\prod_{U,U'} \lambda_k(U) \lambda_{f}(U') } 
\end{bmatrix}. 
\]
Note that the $U$ sets are the same for all the sources, but we're summing over all possible values for each pair. 

We wish to bound $\|b^*-\hat{b}\|$, where $\hat{b}$ is the version of $b^*$ obtained with the use of the estimated $\hat{O}$. We do this for $j,k$ to write
\begin{align*}
|b^*_{j,k} - \hat{b}_{j,k}| &= \left\lvert\left(\sum_{w,z \text{ same sign}} P(\lambda_j = w, \lambda_k = z) - \sum_{w,z \text{ opp. sign}} P(\lambda_j = w, \lambda_k = z) \right)  \right. \\ 
\qquad \qquad \qquad &- \left. \left(\sum_{w,z \text{ same sign}} \hat{c}_{w,z} - \sum_{w,z \text{ opp. sign}} \hat{c}_{w,z} \right) \right\rvert.
\end{align*}
Here we broke up the product over the sum of terms that multiply to $1$ and those that multiply to $-1$. On the estimated side, we use the corresponding values of $\hat{c}$, which are our empirical estimates of the means of the $c$ indicators. Now, we upper bound by moving the sum out, to get
\begin{align*}
|b^*_{j,k} - \hat{b}_{j,k}| \leq \sum_{w,z} |P(\lambda_j=w, \lambda_k=z) - \hat{O}_{wz}|.
\end{align*}
Summing over all the sources, we get that
\[\|b^* - \hat{b}\|_1 \leq \|O^* - O\|_1.\]
From this, we have that
\begin{align}
\|b^* - \hat{b}\| \leq \sqrt{3o} \|O^* - O\|.
\label{eq:bo}
\end{align}

Now we have control over the gap between $b$ and $b^*$. Recall that we form $\hat{q}$ from $\log(\hat{b}^2)$, then we solve the $3 \times 3$ system in \eqref{eq:system3}. Let $\Delta_b = \hat{b} - b^*$. We have that, with the summation below running over the three pairs starting with $(j,k)$,
\begin{align*}
	\|\hat{q} - q^*\|^2
	&=
	\sum_{(j,k)} \left( \log(\hat{b}_{j,k}^2) - \log((b^*)_{j,k}^2) \right)^2 \\
	&=
	4 \sum_{(j,k)} \left( \log(|\hat{b}_{j,k}|) - \log(|(b^*)_{j,k}|) \right)^2 \\
	&= 
	4 \sum_{(j,k)} \left( \log(|b^*_{j,k} + (\Delta_b)_{j,k}|) - \log(|b^*_{j,k}|) \right)^2 \\
	&=
	4 \sum_{(j,k)} \left[ \log \left(1 + \left|\frac{(\Delta_b)_{j,k}}{b^*_{j,k}} \right| \right)\right]^2  \\
	&\leq 
	4 \sum_{(j,k)} \left( \frac{|(\Delta_b)_{j,k}|}{|b^*_{j,k}|} \right)^2  \\
	&\leq 
	\frac{4}{(b^*_{\min})^2} \sum_{(j,k)} (\Delta_b)_{j,k}^2 .
\end{align*}	
Here we used the fact that $(\log(1+x))^2 \leq x^2$. Next, we sum and take square roots and plug in our bound \eqref{eq:bo}
\begin{align*}
\|\hat{q} - q^*\| &\leq \frac{2}{b^*_{\min}} \|\Delta_b\| \\
&\leq \frac{2\sqrt{3o}}{b^*_{\min}} \|O^* - O\|.
\end{align*}

Next, we recall that $\hat{\rho} = \frac{1}{2} + \exp(\frac{\hat{\ell}}{2})$ and similarly for $\rho^*$. Here, the exponent is taken by entry. To obtain $\hat{\ell}$, we solve our system $M \hat{\ell} = \hat{q}$. Then, we have that
\begin{align*}
\| \hat{\rho} - \rho^*\| &= \left\lVert \exp\left(\frac{\hat{\ell}}{2}\right)  -  \exp\left(\frac{\ell^*}{2}\right) \right\rVert \\
&=    \left\lVert   \exp\left(\frac{\ell^*}{2}\right) \left( \exp\left(\frac{ \hat{\ell} - \ell^*}{2}\right) - 1\right) \right\rVert \\
&\leq   \left\lVert  \exp\left(\frac{\ell^*}{2}\right) \right\rVert \left\lVert  \exp\left(\frac{ \hat{\ell} - \ell^*}{2}\right) - 1  \right\rVert \\
&=   \left\lVert  \rho^* \right\rVert \left\lVert  \exp\left(\frac{ \hat{\ell} - \ell^*}{2}\right) - 1  \right\rVert \\
&\leq \sqrt{3}  \left\lVert  \exp\left(\frac{ \hat{\ell} - \ell^*}{2}\right) - 1  \right\rVert.
\end{align*}

If $x$ is small, the we have that $\exp(x)  \leq 2x + 1$. So, for large enough $n$, and thus the case of small $\hat{\ell} - \ell^*$, 
\[\left\lVert \exp\left(\frac{ \hat{\ell} - \ell^*}{2}\right) - 1 \right\rVert  \leq \| \hat{\ell} - \ell^* \|.\]

Thus,
\begin{align*}
\| \hat{\rho} - \rho^*\| &\leq \sqrt{3}  \| \hat{\ell} - \ell^* \| \\
&\leq \sqrt{3}  \|M^\dagger\| \|\hat{q} - q^* \| \\
&\leq \frac{6\sqrt{o}}{b^*_{\min}}  \|M^\dagger\| \|O^* - O\|.
\end{align*}

Since $\|x\|_\infty \leq \|x\|$, we also get that
\begin{align*}
\| \hat{\rho} - \rho^*\|_\infty &\leq  \frac{6\sqrt{o}}{b^*_{\min}}  \|M^\dagger\| \|O^* - O\|.
\end{align*}

This concludes our error analysis for the first system; we proceed to the second. Recall that we assemble the vector $r$ by stacking together $o = 2^T$ entries of various $\rho$'s. Thus, 
\begin{align*}
\|r^* - \hat{r}\| \leq \sqrt{o} \| \hat{\rho} - \rho^*\|_\infty \\
\leq \frac{6o}{b^*_{\min}}  \|M^\dagger\| \|O^* - O\|.
\end{align*}

Next, we deal with the second system: $B_{2T} e = r$ means $e = B_{2T}^{-1}r$, as $B_{2T}$ is full rank, so
\begin{align*}
\| \hat{e} - e^*\| &=  \|B_{2T}^{-1} (\hat{r} - r^*)\| \\
&\leq \|B_{2T}^{-1}\|  \| \hat{r} - r^*\| \\
&\leq  \|B_{2T}^{-1}\|  \frac{6o}{b^*_{\min}}  \|M^\dagger\| \|O^* - O\|.
\end{align*}

All of this was for a fixed source, and we have $m$ such sources. Now, forming $\mu$ from the terms in $e$ only involves scaling by the probabilities of the $Y$'s; the smallest such term is $p_\text{min}$. We have that 
\[\| \hat{\mu} - \mu^* \| \leq  \sqrt{m} \|B_{2T}^{-1}\| \frac{6o}{p_{\min} b^*_{\min}}  \|M^\dagger\| \|O^* - O\|.
 \]

Taking expectations, we have that 
\[\E{}{ \| \hat{\mu} - \mu^* \| } \leq \sqrt{m} \|B_{2T}^{-1}\| \frac{6o}{p_{\min} b^*_{\min}}  \|M^\dagger\| \left(\sqrt{\frac{18 o^2 \log(6o)}{n}} + \frac{2o \log(6o)}{n}\right).\]

Recall that $o = 2^{T}$, we get

\[\E{}{ \| \hat{\mu} - \mu^* \| } \leq 2^{2T} \sqrt{m} \frac{6}{p_{\min}b^*_{\min}} \|B_{2T}^{-1}\|  \|M^\dagger\| \left( \sqrt{\frac{18 \log(6\times 2^T)}{n}} + \frac{2\log(6 \times 2^T)}{n}\right). 
\]

Now we have the general expression. In the case of parameter reduction, we take $T=1$ in the powers, since all of our sources only use a single step, but now we use up to $mT$ of them, which replaces $m$. We then have,
\[
\E{}{ \| \hat{\mu} - \mu^* \| } \leq \sqrt{mT} \frac{24}{p_{\min} b^*_{\min}} \|B_{2T}^{-1}\|  \|M^\dagger\| \left( \sqrt{\frac{18 \log(12)}{n}} + \frac{2\log(12)}{n}\right). 
\]

The expressions for $\phi$ follow similarly, but with pairs of edges, so that we replace $\sqrt{m}$ with $m$, and similarly in $T$.
\end{proof}

We also note that it is possible to derive a matching lower bound reveals the same scaling in $T$ for the general case. This can be done by noting that for worst-case distributions, the matrix Bernstein inequality is sharp \cite{Tropp15}, and, as a result, passing through our linear systems, we can lower bound the resulting error. 

\subsection{Identifiability}
Next we discuss identifiability for our models. Our algorithm already implicitly guarantees identifiability under our assumptions, but we may be interested in which situations are sufficient, in the challenging multi-resolution setting with many latent labels, to guarantee model identifiability in general.

Our approach is to apply results from the seminal work \cite{Allman}, which provides identifiability results based on Kruskals' theorem on the uniqueness of 3-tensor decompositions. Applying this result in creative ways, \cite{Allman} recovers results on identifiability for latent mixtures of product distributions, hidden Markov models (HMMs), and other latent variable models. The sense of identifiability here is \emph{generic identifiability}, which corresponds to the information-geometric view. We consider the parameter space of our models as a variety, and demonstrate identifiability everywhere except potentially a measure-zero subvariety. 

We break down our approach into two cases. The first case does not have any type of parameter reduction. The second case does. In both cases, we consider the parametrizations of both the true label and the sources. As usual, our sequence is of length $T$.

\paragraph{General case} 
First, we consider the case where the distribution of true labels does not factorize across time. We consider labels for all subsets of steps of lengths 1 ($Y_1$ resolution, or frames), length $T$ (the full-sequence level) and one additional resolution window level, such $g$, with $1<g<T$. 
We assume that the labels are in $\{0,\ldots, r-1\}$, so that our label alphabet is  \[\mathcal{Y} = r^{T+1+(T-g+1)}.\]

Next, we consider $m$ sources that vote on (some part) of the time series sequence $X$. We view these as being independent conditioned on ${Y}$ (we can group non-independent sources together if necessary and simply count the remaining components). 

Let us say that each supervision source is capable of producing one of $v$ possible votes on the time series. One concrete example is sources that can label each window of length up to $w$ (with each label having $r$ choices). Then the total number of possible votes is \[v = r^T + r^{T-1} + \ldots + r^{T-w+1} = \frac{r^{T+1}-r^{T+1-w}}{r-1}.\]

Now we can apply Corollary 5 from \cite{Allman}, which states that identifiability is guaranteed if \[m \geq 2 \lceil \log_v r^{2T+2-g} \rceil + 1.\]

For an example of how this works, consider $r=2$, so that each subset gets a binary label, as throughout our paper. Then, $v = 2^{T+1} - 2^{T+1-w}$. 
If our window size is just 1, we get $v = 2^T$, and then, if, say $g \geq 2$, that $ \log_v 2 \lceil \log_v 2^{2T+2-g} \rceil + 1 = 2 \times 2 + 1 = 5$, so that we need 5 sources.

\paragraph{Parameter Reduction}
The previous approach is challenging practically, as we showed in our theorem in the previous section. The parameter space is very large---which makes parameter recovery challenging even if identifiability is ensured. In fact there is a tension between identifiability and recovery, since the first requires a large number of parameters, while the second is easier with fewer. Below we describe identifiability in the general case using based on parameter reduction. We exploit a type of reduction to the HMM model.

First, we use a Markovian model for $Y = (y_1, \ldots, y_{2T+1})$. The $2T+1$ is simply for convenience here. We group together windows of length $u$. For convenience, say $u | (2T+1)$. Then, let each of the groups $(y_1, \ldots, y_u), (y_{u+1}, \ldots, y_{2u}), \ldots$ form a stationary Markov chain. The state space corresponds to the space of labels of subsets of the window, which has cardinality $r^{2^u}$---we allow labels for any particular subset of the window. 

Next, we similarly set up a model for each of the sources. We set the number of outputs for each window of length $u$ to again be $v$. Again, if we have $m$ conditionally independent sources, our alphabet over the sources is $v^m$. Moreover, the Kruskal rank of the product distribution matrix of the sources has full Kruskal rank if each of the functions does. This puts us in a position to apply Theorem 6 from \cite{Allman}, for the HMM-style model we've just defined. We need, for identifiability, that the number of windows we label, which is $k = (2T+1)/u$, satisfies
\[ \binom{k + v^m - 1}{v^m - 1} \geq r^{2^u}.\]

We can write this as a function of $T$:
\begin{align*}
&\binom{(2T+1)/u + v^m - 1}{v^m - 1} \geq r^{2^u}.
\end{align*}
Now we can express this in terms of particular variables; note that this is a tradeoff between:
\begin{itemize}
    \setlength{\itemsep}{0pt}
	\setlength{\parskip}{0pt}
	\item $u$, the complexity of the chain for $Y$,
	\item $T$, the parameter for the length of the sequence, 
	\item $v$, the resolution of the supervision source votes, and
	\item $m$, the number of conditionally independent sources.
\end{itemize} 

\subsection{Simulations}
Simulations of effects of increasing $N$ and parameter tying on estimation
error, effects of modeling sequential dependencies on prediction performance,
and runtime are shown in Figure~\ref{fig:sims}.
\begin{figure}
    \includegraphics[width=\columnwidth]{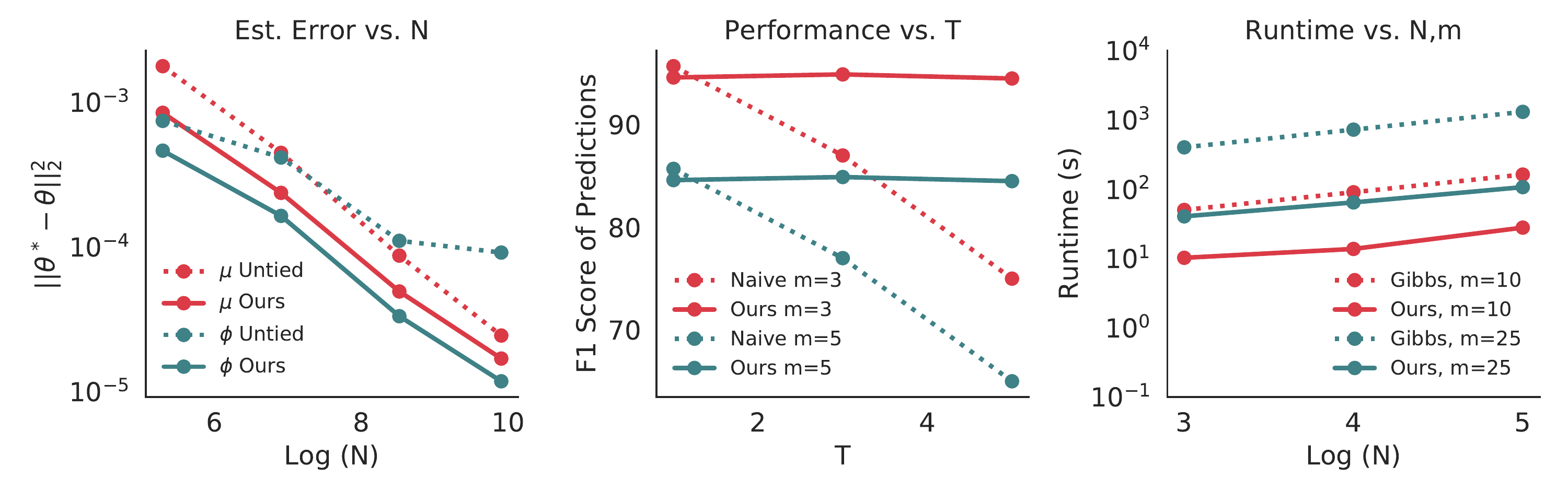}
    \caption{Simulation plots. (Left) Estimation error $||\hat{\mu} - \mu^*||^2$ decreases with increasing $N$ and improves with parameter tying (Ours). (Middle) Modeling sequential dependencies (Ours) leads to improved prediction performance over naive model. (Right) Our runtime is up to $90\times$ faster than Gibbs-sampling based approaches.}
    \vspace{-1.0em}
    \label{fig:sims}
  \end{figure}
  
\subsection{Conditional Independencies}
In Section~\ref{sec:source_accuracy_estimation}, we discussed assumptions that are used in the parameter recovery algorithm. We described that we use a list of independencies among the source agreement measures (that is, the $a_j(U,V)$ terms). These independencies can be derived from the source graphs, which are provided by the user. Below, we give additional details on this operation.

Note that there are multiple types of graphical models that provide us with such independencies. Our method is agnostic to this choice, as long as we can obtain the list of independencies. The illustrative example we consider is that of a binary Ising model where the sources have no singleton potentials. For simplicity of notation, we simply show a single label $Y = Y_1$ and the sources $\lf_1, \ldots, \lf_m$ (that is, we take $\tau_1 =\{1, \ldots, m\}$); normally, we would have a full model over all the $Y_i$'s and $\lf_j$'s. We write the density as
\[f(Y, \lambda_1, \lambda_2, \ldots, \lf_m) = \frac{1}{Z} \exp \left( \theta_{Y} Y + \sum_{i=1}^m \theta_{Y,i} Y\lf_i + \sum_{(i,j) \in E} \theta_{i,j} \lf_i \lf_j\right) .\]
Here, $Z$ is the partition function, $E$ is the edge set among the sources in the graphical model, and the $\theta$'s are canonical parameters. The following argument provides the intuition for why the $a_j(U,V)$ are independent in this setting whenever their nodes are disconnected in the graph. We have that 
\begin{align*}
f(Y, \lambda_1, \lambda_2, \ldots, \lf_m) &= \frac{1}{Z} \exp \left( \theta_{Y} Y + \sum_{i=1}^m \theta_{Y,i} Y\lf_i + \sum_{(i,j) \in E} \theta_{i,j} \lf_i \lf_j\right) \\
&=  \exp \left( \theta_{Y} Y + \sum_{i=1}^m \theta_{Y,i} Y\lf_i + \sum_{(i,j) \in E} \theta_{i,j} (\lf_i Y) (\lf_j Y) \right) \\
&=  \exp \left( \theta_{Y} Y + \sum_{i=1}^m  \theta_{Y,i}  a_i(U,V) + \sum_{(i,j) \in E} \theta_{i,j} a_i(U,V) a_j(U,V) \right),
\end{align*}
which indeed factorizes as long as there is no path in $G_{\text{source}}$ between $i$ and $j$ (other than the one through $Y$). Next, these terms are summed to produce a distribution over the $a_j$'s, and \emph{symmetries} enable us to produce the desired independencies.

The above showed the simplified case where $U$ is a single source and $V$ consists of the $Y$ label only. This can be extended to the case where $|U| > 1$ or $|V| > 1$ (or both). There is a parity requirement for the symmetries to work. Specifically, at least one of the $a_i, a_j$ must involve an even number of terms, that is, $|U| + |V|$ is even.  

\section{Extended Experimental Details}
\label{sec:appendix}
We describe additional details about the tasks described in Section~\ref{sec:exp}, including details about supervision sources, the user-defined class prior, and the end model trained on labels generated by baselines and our method. Dataset statistics provided in Table~\ref{table:stats}.

\subsection{Dataset Details and Train/Dev/Test Splits}
\textbf{Bicuspid Aortic Valive (BAV)}:
We use the dataset from~\cite{fries2018weakly} and use the train/dev/test
splits from that work.

\textbf{Interview Detection (Interview)}:
We use the dataset from~\cite{fu2019rekall} and use the dev/test splits from
that work. We additionally use an additional $57$ hours of unlabelled data as
the train split.

\textbf{Freezing Gait (Gait)}:
Our dataset consists of sensor data sessions from different patients.
We reserve a collection of sessions for dev and test, and split by patient to
ensure that dev and test come from similar distributions.

\textbf{Movie Shot Detection (Shot)}:
Our total dataset consists of 589 Hollywood movies (roughly 1200 hours).
We treat windows of $16$ consecutive frames as elements in our sequence notation
and classify individual elements with the end model (so a sequence of length
five is $80$ consecutive frames).
We have ground truth annotations for 45 minutes of data randomly distributed
across 29 of the movies, which we split into the dev and test set by scene.
We note that this gives an unfair advantage to traditional supervision, but
\sn{} outperforms nonetheless. 

\textbf{ActivityNet Basketball Identification (Basketball)}:
We take the subset of ActivityNet videos containing sports videos and aim to
identify basketball videos.
We sample one frame every two seconds and classify individual frames as either
coming from basketball videos or the other sports videos.
We randomly select $5\%$ of the videos as our dev set and $5\%$ of the videos
as our test set.

\textbf{Cyclist Detection in Self-Driving Car Dataset (Car)}:
Our dataset consists of 50 minute of self-driving car dash cam footage, split
into five videos \cite{cydet}.
The task is to identify whether individual frames, sampled at 10 FPS, contain
cyclists.
We select one video to split into dev and test, and reserve the rest of the
videos as unlabelled training data.
We split dev/test by taking a strided window of $5$ frames with a stride of $10$
frames (starting dev and test at frames $1$ and $6$ respectively) to ensure
that dev and test come from the same distribution.

\textbf{Disorder Tagging in Electronic Health Record Text (EHR)}: This dataset consists of 299 patient notes sampled from MIMIC-III\cite{mowery2014task, johnson2016mimic}, labeled for all mentions of disorders (e.g., aortic stenosis, pneumonia). 
The dataset is split into 133 development documents and 166 test documents, containing 10,940 and 16,641 sentences respectively.
Training data consists of 10,000 sentences randomly sampled from 5,000 unlabeled MIMIC-III documents.
Labels are generated per-word in IO (inside/outside) tag format.

%
%
\subsection{Task-Specific End Models}
For BAV and Shot tasks, we use previously published end model architectures.  
For Gait and Open tasks, we rely on generic, off-the-shelf architectures commonly used for these modalities.  
We do not claim that these end models achieve the best possible performance for these tasks; our goal is to compare the relative improvements that our sequential weak supervision model provides compared to other baselines, which is orthogonal to achieving state-of-the-art performance for these specific tasks.

\textbf{Bicuspid Aortic Valve (BAV)}: We use the CNN-LSTM architecture described in \cite{fries2018weakly} for use in classifying aortic valve malformations.
This architecture includes a frame encoder for learning frame-level features and a sequence encoder for combining individual frames into a single feature vector. 
The frame encoder is a Dense Convolutional Network (DenseNet) \cite{Huang2017-fw} with 40 layers and a growth rate of 12, pretrained on 50,000 images from CIFAR-10 \cite{Krizhevsky2009-zf}. 
The sequence encoder is a bidirectional Long Short-term Memory (LSTM) \cite{Hochreiter1997-fj} with soft attention \cite{Xu2015-eo}.
All weights were fine-tuned during training. Models are trained using all MRI frames as input. 

\textbf{Interview Detection (Interview)}: We use ResNet-50 pre-trained on ImageNet to classify individual frames of the video.

\textbf{Freezing Gait (Gait)}: We use a single layer bidirectional LSTM and hidden state dimension 300 as our end model that takes in a multivariate sensor stream as input. In order to provide longer sequential context, we pass in a windowed version of each candidate that includes past and future frames. Window size was tuned empirically, with {[}-3,+1{]} performing best overall. Since the sequence length of each frame slightly varies, we then pad these sequences (with 0's) and truncate any sequences over a pre-defined maximum sequence length. To provide more contextual signal, we also add multiplicative attention to pool over the hidden states in the LSTM.

\textbf{Movie Shot Detection (Shot)}: We use a C3D ConvNet with a ResNet-18
backbone pre-trained on the Kinetics dataset \cite{hara3dcnns}.
This is a common architecture for deep shot detection \cite{hassanien2017deepsbd, tang2018dsm}.
We feed in 16 conseuctive frames as input and classify whether or not there is
a shot boundary in the 16 frames. 

\textbf{ActivityNet Basketball Identification (Basketball)}:
We use ResNet-18 pre-trained on ImageNet to classify individual frames of the
video.

\textbf{Cyclist Detection in Self-Driving Car Dataset (Car)}:
We only report label model results for this dataset, since we cannot release
the proprietary end model used for this task.

\textbf{Disorder Tagging in Electronic Health Record Text (EHR)}: We only report label model results for this dataset.

%
%
\subsection{Supervision Sources}
Supervision sources are expressed as Python functions with an average of 5 lines each. 
We list how many of the supervision sources operated on an element-level basis, a subsequence level basis (more than one frame), and a sequence level basis in Table~\ref{table:stats}. 
The supervision sources relied on the following information to assign noisy labels:

\textbf{Bicuspid Aortic Valve (BAV)}: First, the aortic valve in each frame was segmented using an intensity-based thresholding technique. 
The supervision sources relied on feature values derived from these segmented regions (i.e., area, perimeter, average intensity, eccentricity, and ratio of area and perimeter) to assign labels to each frame. 
Each supervision source assigned a label to the same frame that it used information from.

\textbf{Interview Detection (Interview)}: Two weak supervision sources use face identities on individual frames; they vote yes if Bernie Sanders or a host are detected in a particular frame, respectively.
One labeling function checks whether the text ``thank you" appears in the transcript within $30$ seconds of a segment, and another checks whether there are the same number faces over the course of $30$ seconds.

\textbf{Freezing Gait (Gait)}: The first supervision source employed uses stride time arrhythmicity \citep{plotnik2005freezing, plotnik2007new}, which is calculated as average coefficient of variation for the past 3 stride times of the left and right leg.
In addition to stride time arrhythmicity, other supervision sources we use involve the swing angular range of the shank, and the amplitude and variance in shank angular velocity.
Out of the five total supervision sources used for this task, 3 of them operated on an element-level basis, and 2 of them operated on a sequence level (4 and 11 frames at a time). 

\textbf{Movie Shot Detection (Shot)}: We compute frame-to-frame differences in HSV, RGB, and optical flow histograms.
We detect frames that have large amounts of visual change from the frames
immediately preceding them by detecting outliers in the frame-to-frame
differences between histograms.
These make up our three sequence-level weak supervision sources.
We also introduce two weak supervision sources based on face detections.
We run the MTCNN face detector \cite{zhang2016joint} twice a second (once every
twelve frames for a film shot at 24 FPS) and say that there is no shot change
between detections if we find the same number of faces or if we find faces in
the same location.
On the other hand, if we find faces in different locations between detections,
we say that there is a shot change.
These make up our two subsequence-level weak supervision sources.

\textbf{ActivityNet Basketball Identification (Basketball)}:
We use an off-the-shelf object detector \cite{yolov3} on one frame every two
seconds to generate primitives.
Our weak supervision sources operate on the objects detected in each frame; we
detect whether a person and ball are detected in the frame, what the distance
between the person and ball in the frame are, the color of the ball, and how
much vertical distance the ball moves across  a sequence.

\textbf{Cyclist Detection in Self-Driving Car Dataset (Car)}:
We aim to classify whether frames contain cyclists in a representative sample
of a self-driving car dataset.
We detect whether frames have people or bicycles using an off-the-shelf object
detector \cite{massa2018mrcnn}.
The object detect small bikes (i.e. when the bikes are far away), so we also
write some heuristics for small person detections.

\textbf{Disorder Tagging in Electronic Health Record Text (EHR)}:
Supervision sources are a collection of biomedical lexicons from the Unified Medical Language System (UMLS) \cite{bodenreider2004unified} and a single stopword list. 
UMLS lexicons are broken down by semantic type (e.g., Disease or Syndrome, Finding) with each type mapped to a positive or negative label.
Positive and negative lexicons are merged by source vocabulary (e.g., SNOMEDCT\_US) to generate 12 supervision sources.


\begin{table}[ht!]
    \centering
    \begin{tabular}{@{}lccccccccc@{}}
        \toprule
                         & \multicolumn{1}{l}{} & \multicolumn{4}{c}{\textbf{Dataset Statistics}}                                                 & \multicolumn{4}{c}{\textbf{Supervision Statistics}}                                       \\ \cmidrule(l){3-10} 
        \textbf{Dataset} & \textbf{End Model} & \textbf{T} & \textbf{$N_{train}$} & \textbf{$N_{dev}$}   & \textbf{$N_{test}$}  & \textbf{m} & \textbf{$R_1$} & \textbf{$R_2$} & \textbf{$R_3$} \\ \midrule
        BAV              & CNN-LSTM           &  5         & 4329                 & 10 6                 & 94                   & 5          & 5               & 0              & 0              \\
        Interview    & ResNet-50      &  5        & 6835                  & 3026                   & 3563                   & 4          & 2               & 2              & 0              \\
        Gait             & LSTM               &  3         & 1793                 & 630                  & 1014                 & 5          & 3               & 0              & 2              \\
        Shot             & C3D ConvNet        &  5         & 35,376               & 363                  & 369                  & 5          & 0               & 2              & 3              \\
        Basketball       & ResNet-18          &  5         & 3594                 & 212                  & 244                  & 4          & 3               & 0              & 1              \\ 
        Car              & -                  &  5         & 4785                 & 670                  & 669                  & 4          & 4               & 0              & 0              \\ 
        EHR              & -                  &  2         & 10,000               & 10,940               & 16,641               & 12         & 12              & 0              & 0              \\ \bottomrule
        \end{tabular}
    \caption{We report the train/dev/test split in terms of the number of sequences in each set. The dev and test set have ground truth labels, which we assign labels to the train set using our method or one of the baseline methods. $m$ is the number of supervision sources, with $R_1$, $R_2$, and $R_3$ the number of supervision sources that label individual elements, subsequences, or the whole sequence, respectively.}
    \label{table:stats}
    \end{table}

\subsection{User-Defined Class Prior}
We set the task-specific class prior for the tasks in the following manner. As discussed in Section~\ref{sec:exp}, the user-defined prior outperformed the uniform and development set based for Open and Shot, but not for BAV and Gait.

\textbf{Bicuspid Aortic Valve (BAV)}: The labels were assigned on a sequence level for this task. 
The estimated incidence of BAV in the population is $1$-$2\%$, which was used to set the user-defined prior. 
However, the development and test sets have much higher prevalence rates ($6$-$7\%$) as an artifact of their construction, thus an empirical prior derived directly from the development set performed best overall.
Note the uniform class balance performs poorly due to incorrectly assuming that all label combinations within a sequence are equally likely, i.e.,  frames \emph{within} a single sequence can alternate between BAV and normal. 

\textbf{Interview Detection (Interview)}: The class priors were set starting with the class balance from the development set, and then slightly adjusting probabilities based on intuition of interview incidence.

\textbf{Freezing Gait (Gait)}: We design our own class prior by first starting with the class balance from the validation set, and then slightly adjusting probabilities based on intuition of freezing behavior. For example, we don't expect freezing and non-freezing behavior to alternate frequently in successive gait cycles so we assign very low probabilities to these events. Other sequences, such as consecutive freezing and consecutive non-freezing are likely more common, so we assign relatively high probabilities to these events.

\textbf{Movie Shot Detection (Shot)}: In this task, the labels are assigned to individual candidates in a
five-candidate sequence.
Each candidate is in turn a window of 16 consecutive frames.
We set our prior based on the development set, but we manually reduce the
likelihood of rare sequences to $0\%$ (in particular, we set the likelihood of
a sequence to $0\%$ if we observe five or fewer instances in our development
set).

\textbf{ActivityNet Basketball Identification (Basketball)}:
In this task, labels are assigned per-frame for the end model, but our sequential
modeling views sequences of five frames.
We set our prior based on the development set, then slightly adjust the values
based on intuition.

\textbf{Cyclist Detection in Self-Driving Car Dataset (Car)}:
We set our prior based on the development set, but manually reduce the likelihood
of rare sequences to $0$ (in particular, we set the likelihood of a sequence to
$0$ if we observe five or fewer instances in our development set).

\textbf{Disorder Tagging in Electronic Health Record Text (EHR)}:
We set our prior based on the development set.

\subsection{Detailed Results}
\label{sec:appendix_results}
We report detailed precision, recall, and F1 results for all datasets in
Tables~\ref{table:p_r_f1_baselines} and~\ref{table:p_r_f1_ablations}.

\begin{table}[htbp]
    \centering
    \small
    \begin{tabular}{@{}lrcccccccccc@{}}
        \toprule
        \multicolumn{1}{c}{} &                 & \multicolumn{3}{c}{\textbf{Baselines}}            & \\ \cmidrule(l){3-6}
        \textbf{Task}        & \textbf{Metric} & \textbf{TS}     & \textbf{MV}    & \textbf{DP}    & \textbf{\sn}\\ \midrule
                            & Precision        & 26.1 $\pm$ 3.8  & 6.9 $\pm$ 8.6  & 70.0 $\pm$ 19.8 & \textbf{100.0 $\pm$ 0.0} & \\
         BAV                & Recall           & 20.0 $\pm$ 7.0  & 5.7 $\pm$ 7.0  & 45.7 $\pm$ 5.7  &  \textbf{37.1 $\pm$ 7.0}  & \\
                            & F1               & 22.1 $\pm$ 5.1  & 6.2 $\pm$ 7.6  & 53.2 $\pm$ 4.4  & \textbf{53.8 $\pm$ 7.6}  & \\ \hline
                            & Precision        & 72.4 $\pm$ 4.0  & 48.7 $\pm$ 5.7 & 4.5 $\pm$ 0.1 & \textbf{89.6 $\pm$ 4.2} & \\
         Interview          & Recall           & 89.5 $\pm$ 3.0  & 72.4 $\pm$ 6.4 & \textbf{99.1 $\pm$ 0.0}  & 94.6 $\pm$ 0.8  & \\
                            & F1               & 80.0 $\pm$ 3.4  & 58.0 $\pm$ 5.3 & 8.7 $\pm$ 0.2  & \textbf{92.0 $\pm$ 2.2}  & \\ \hline
                            & Precision        & 65.2 $\pm$ 13.7 & 47.0 $\pm$ 1.0 & 50.3 $\pm$ 1.6 & \textbf{65.6 $\pm$ 1.5} & \\
         Gait               & Recall           & 47.6 $\pm$ 28.1 & \textbf{89.8 $\pm$ 3.0} & 84.1 $\pm$ 2.2 & 70.8 $\pm$ 2.4  & \\
                            & F1               & 47.5 $\pm$ 14.9 & 61.6 $\pm$ 0.4 & 62.9 $\pm$ 0.6 & \textbf{68.0 $\pm$ 0.7}               & \\ \hline
                            & Precision        & 87.7 $\pm$ 2.5  & 79.7 $\pm$ 2.1 & 79.0 $\pm$ 1.9 & \textbf{87.8 $\pm$ 2.9}   & \\ 
         Shot               & Recall           & 79.3 $\pm$ 1.3  & 93.4 $\pm$ 1.0 & \textbf{94.3 $\pm$ 1.2} & 87.6 $\pm$ 3.4   & \\
                            & F1               & 83.2 $\pm$ 1.0  & 86.0 $\pm$ 0.9 & 86.2 $\pm$ 1.1 & \textbf{87.7 $\pm$ 1.0}   & \\ \hline
                            & Precision        & 30.3 $\pm$ 3.6  & 10.0 $\pm$ 6.9 & 7.6 $\pm$ 2.9  & \textbf{33.0 $\pm$ 4.0}   & \\
         Basketball         & Recall           & 24.1 $\pm$ 0.4  & 6.8 $\pm$ 4.4  & 8.0 $\pm$ 3.7  & \textbf{46.0 $\pm$ 7.2}   & \\
                            & F1               & 26.8 $\pm$ 1.3  & 8.1 $\pm$ 5.4  & 7.7 $\pm$ 3.3  & \textbf{38.2 $\pm$ 4.1}   & \\ \hline
                            & Precision        & N/A             & 57.7           & 57.8           & \textbf{95.3}   & \\
         Car$^*$            & Recall           & N/A             & 81.3           & \textbf{83.7}  & 64.9            & \\
                            & F1               & N/A             & 67.5           & 68.4           & \textbf{77.3}   & \\ \hline                             
                            & Precision        & N/A             & 82.7           & 79.1           & \textbf{85.6}   & \\
         EHR$^*$            & Recall           & N/A             & 57.4           & 61.2           & \textbf{61.4}   & \\
                            & F1               & N/A             & 67.8           & 69.0           & \textbf{71.5}   & \\ \bottomrule                                            
        \end{tabular}
        \caption{Precision, recall, and F1 numbers for baselines.
        All reported values are means across five random weight initializations,
        $\pm$ standard deviation, except for the \textbf{Car} and \textbf{EHR} task, where we
        only report label model performance.}
    \label{table:p_r_f1_baselines}
\end{table}


\begin{table}
    \centering
    \small
    \begin{tabular}{@{}lrcccccccc@{}}
        \toprule
        \multicolumn{1}{c}{} &                 & \multicolumn{2}{c}{\textbf{Parameter Ablations}}  &  \multicolumn{3}{c}{\textbf{Class Prior}}                   \\ \cmidrule(l){3-4} \cmidrule(l){5-7}
        \textbf{Task}     & \textbf{Metric} & \textbf{w.o Param Tie}  & \textbf{w.o Temp Deps}  &  \textbf{Uniform} & \textbf{Dev}            & \textbf{User}           \\ \midrule
                          & Precision       &  75.3 $\pm$ 13.6        &  42.5 $\pm$ 17.0        &  24.7 $\pm$ 13.3    & \textbf{100.0 $\pm$ 0.0}         & 99.8 $\pm$ 0.5                        \\
          BAV             & Recall          &  42.9 $\pm$ 9.0         &  \textbf{45.7 $\pm$ 5.7}&  17.1 $\pm$ 5.7    & 37.1 $\pm$ 7.0 & 34.3 $\pm$ 14.6    \\
                          & F1              &  53.0 $\pm$ 5.7         &  41.9 $\pm$ 6.7         &  20.0 $\pm$ 8.3    & \textbf{53.8 $\pm$ 7.6} & 48.1 $\pm$ 14.5  \\ \hline
                          & Precision       &  86.7 $\pm$ 7.0         &  80.8 $\pm$ 5.3         &  73.2 $\pm$ 3.3    & \textbf{89.6 $\pm$ 4.2} & 88.4 $\pm$ 3.4                        \\
          Interview       & Recall          &  92.2 $\pm$ 2.3         &  88.4 $\pm$ 7.3         &  94.4 $\pm$ 0.0  & \textbf{94.6 $\pm$ 0.8} & 91.4 $\pm$ 6.8    \\
                          & F1              &  89.2 $\pm$ 3.6         &  84.2 $\pm$ 3.8         &  82.4 $\pm$ 2.0    & \textbf{92.0 $\pm$ 2.2} & 89.8 $\pm$ 4.9  \\ \hline
                          & Precision       & 49.7 $\pm$ 3.8          &  66.5 $\pm$ 2.1         &  65.6 $\pm$ 1.5   & 57.9 $\pm$ 2.9          & \textbf{67.9 $\pm$ 6.3} \\
          Gait            & Recall          & 74.7 $\pm$ 9.5          &  64.7 $\pm$ 5.7         &  70.8 $\pm$ 2.4   & \textbf{80.5 $\pm$ 3.7}          & 66.3 $\pm$ 7.5       \\
                          & F1              & 59.5 $\pm$ 5.4          &  65.3 $\pm$ 2.2         &  \textbf{68.0 $\pm$ 0.7}   & 67.2 $\pm$ 0.6          & 66.3 $\pm$ 1.1       \\ \hline
                          & Precision       & 84.3 $\pm$ 3.2          & 82.3 $\pm$ 1.8          &  60.9 $\pm$ 6.5   & 70.8 $\pm$ 2.3          & \textbf{87.8 $\pm$ 2.9} \\ 
          Shot            & Recall          & 88.6 $\pm$ 2.1          & 91.4 $\pm$ 1.1          &  91.9 $\pm$ 6.3   & \textbf{95.0 $\pm$ 1.1} & 87.6 $\pm$ 3.4          \\
                          & F1              & 86.6 $\pm$ 1.3          & 86.6 $\pm$ 0.5          &  72.9 $\pm$ 3.4   & 81.1 $\pm$ 1.2          & \textbf{87.7 $\pm$ 1.0} \\ \bottomrule
                          & Precision       & 24.2 $\pm$ 3.4          & 5.5 $\pm$ 0.5           &  10.0 $\pm$ 0.0   & 24.1 $\pm$ 11.2         & \textbf{33.0 $\pm$ 4.0} \\ 
          Basketball      & Recall          & \textbf{51.6 $\pm$ 9.0} & 45.1 $\pm$ 4.1          &  100.0 $\pm$ 0.0  & 33.5 $\pm$ 22.8         & 46.0 $\pm$ 7.2 \\
                          & F1              & 32.9 $\pm$ 4.9          & 9.9 $\pm$ 0.8           &  18.3 $\pm$ 0.0   & 27.6 $\pm$ 15.4         & \textbf{38.2 $\pm$ 4.1} \\ \bottomrule
        \end{tabular}
        \caption{Precision, recall, and F1 numbers for ablations.}
    \label{table:p_r_f1_ablations}
\end{table}

\subsection{Parameter Ablations}
We examine how the following elements of our method improve empirical performance (Table~\ref{table:p_r_f1_ablations}):

\textit{No parameter reduction (w.o Param Tie):} We force our model to learn a separate accuracy parameter per supervision source per resolution it labels and a separate correlation parameter per pairwise dependency. We show that this can hurt end model performance by \shareimp{} F1 points on average since there is not enough data to correctly estimate this many parameters. 

\textit{No sequential Dependencies (w.o Temp Deps):} We remove all sequential dependencies from our model, but still learn accuracy parameters for the supervision sources with parameter reduction. This hurts end model performance by \depsimp{} F1 points on average since removing the sequential dependencies among the supervision sources leads to ``double counting'' of the votes from sources that use similar information from the underlying data and overestimates the accuracies for these sources.

\subsection{Class Prior Ablations}
We examine the effect of the user-defined prior for the distribution of labels in a sequence (Table~\ref{table:p_r_f1_ablations}):

\textit{Uniform Probability (Uniform)}: All label configurations for a given sequence are equally likely.

\textit{Prior based on Dev Set (Dev)}: Class priors are set empirically using the development set.

\textit{User Defined Prior (User)}: The user defines a class distribution manually, and we provide task specific details in the Appendix. 
For Shot, the user-defined prior improves end model performance by \num{14.8} F1 points compared to uniform prior since shot transitions are rare events. 
Gait performs the best with a uniform prior, which is expected since there is no clear pattern in how freezing occurs while walking.

\end{document}